\documentclass{article}
\usepackage{booktabs} 
\usepackage[margin=1in]{geometry}

\usepackage[ruled, noend]{algorithm2e}
\SetAlFnt{\small}
\SetAlCapFnt{\small}
\SetAlCapNameFnt{\small}
\SetAlCapHSkip{0pt}
\IncMargin{-\parindent}

\SetCommentSty{mycommfont}

\usepackage[utf8]{inputenc}
\usepackage{amssymb}
\usepackage[margin=1in]{geometry}
\usepackage{amsthm}
\usepackage{thmtools}
\usepackage{thm-restate}

\usepackage{amsmath}
\usepackage{comment}
 
\usepackage{amssymb}
\usepackage{tikz}
\usepackage{mathtools}
\usepackage{bbm}

\usepackage{dsfont}
\usepackage{graphicx}
\usepackage{multirow}
\usepackage{color}
\usepackage{xcolor}
\usepackage{natbib}
\usepackage{enumitem}

\usepackage{hyperref}
\usepackage{graphicx}
\usepackage{subfigure}
\usepackage{float}
\hypersetup{%
   breaklinks,%
   colorlinks=true,%
   linkcolor=black,%
   urlcolor=black,%
   citecolor=[rgb]{0,0,0.45}
}

\theoremstyle{plain}
\newtheorem{theorem}{Theorem}[section]

\newtheorem{lemma}[theorem]{Lemma}

\newtheorem{definition}[theorem]{Definition}
\newtheorem{assumption}[theorem]{Assumption}

\theoremstyle{definition}

\DeclareMathOperator*{\argmin}{argmin}
\DeclareMathOperator*{\argmax}{argmax}

\newcommand{\reals}{\mathbb{R}}
\newcommand{\br}{\text{BR}}
\newcommand{\Hs}{\mathcal{H}}
\newcommand{\X}{\mathcal{X}}

\usepackage{color-edits}[showdeletions]

\title{Minimax Group Fairness in Strategic Classification}
\author{Emily Diana \thanks{Carnegie Mellon University (CMU). Email: ediana@andrew.cmu.edu} \and Saeed Sharifi-Malvajerdi \thanks{Toyota Technological Institute at Chicago (TTIC). Email: {saeed@ttic.edu}}  \and Ali Vakilian \thanks{Toyota Technological Institute at Chicago (TTIC). Email: {vakilian@ttic.edu}}}

\date{}

\begin{document}

\maketitle
 
\begin{abstract}
In strategic classification, agents manipulate their features, at a cost, to receive a positive classification outcome from the learner's classifier. The goal of the learner in such settings is to learn a classifier that is robust to strategic manipulations. While the majority of works in this domain consider accuracy as the primary objective of the learner, in this work, we consider learning objectives that have group fairness guarantees in addition to accuracy guarantees. We work with the minimax group fairness notion that asks for minimizing the maximal group error rate across population groups.

We formalize a fairness-aware Stackelberg game between a population of agents consisting of several groups, with each group having its own cost function, and a learner in the agnostic PAC setting in which the learner is working with a hypothesis class $\Hs$. When the cost functions of the agents are separable, we show the existence of an efficient algorithm that finds an approximately optimal deterministic classifier for the learner when the number of groups is small. This algorithm remains efficient, both statistically and computationally, even when $\Hs$ is the set of all classifiers. We then consider cost functions that are not necessarily separable and show the existence of \emph{oracle-efficient} algorithms that find approximately optimal randomized classifiers for the learner when $\Hs$ has finite strategic VC dimension. These algorithms work under the assumption that the learner is fully transparent: the learner draws a classifier from its distribution (randomized classifier) \emph{before} the agents respond by manipulating their feature vectors. We highlight the effectiveness of such transparency in developing oracle-efficient algorithms. We conclude with verifying the efficacy of our algorithms on real data by conducting an experimental analysis.
\end{abstract}

\section{Introduction}
Although traditional machine learning and statistics has focused on cases where the testing and training samples are drawn from the same underlying distribution, there are many application scenarios that differ from this setting. Take, for instance, credit scoring. Although the population may not have the precise equation that is used to calculate their credit score by different agencies, there is a good understanding of the basic factors that contribute to it. Many articles online may weigh the pros and cons of strategies to increase one's credit score -- will the hard pull from applying for a new credit card outweigh the benefit of having more credit? Can adding certain types of loans increase the diversity of your portfolio? Are people actually improving their ability to pay back a loan by \textit{applying for more} loans to diversify their credit history? Does having more credit available actually affect credit \textit{worthiness}? As Goodhart's law is stated in \cite{Strathern_1997}, ``When a measure becomes a target, it ceases to be a good measure.''

Given the knowledge that people will typically try to respond strategically in the face of such measures, can this measurement be designed to be as accurate as possible even in the face of manipulation? Furthermore, can this be done in a way that satisfies some notion of statistical \textit{fairness} across sensitive groups? Given the increasing automation of many decision-making procedures in high-stakes contexts, issues of equity and fairness in machine learning applications are increasingly important. Applications such as voice recognition software, credit lending, college admissions, criminal recidivism, online advertising, and many more can be studied through this lens of fairness and equity. Algorithmic and machine learning fairness is a nuanced field, and a constant challenge which typically involves domain experts and context specific reasoning is how to mathematically define fairness in a given context. Would we like a model that equalizes error rates across groups? Might we prefer one that treats similar individuals similarly? Would we rather have that the model exhibits the same true positive rates between sensitive groups? Even when these objectives have been settled, there are varied approaches to intervening in the model training pipeline. One could examine the quality of data collected (do the data exhibit historical biases or are certain groups underrepresented in the data), one could add constraints to the statistical learning process or choose a bespoke method to specifically optimize for a specified version of fairness, or one could post-process the model to correct imbalances after training has been completed.

In this paper, we tackle the challenge of producing minimax fair models in the strategic classification setting. Minimax group fairness, closely related to Rawlsian fairness and the difference principle, seeks solutions that make the worst off group as well off as possible. In the strategic setting, we are interested in situations in which certain groups are fundamentally more challenging to predict or have a small representation in the dataset, but we are also motivated by settings in which certain groups may have fewer resources with which to adapt to a policy or rule published by a firm. This could be because of financial, educational, cultural, or time availability reasons, among many others. For example, time and financial resources may allow individuals from one group to afford a SAT tutor and dramatically increase their score on the SAT without significantly increasing their college readiness, whereas individuals from another group with the same college readiness may have limited time to study on their own and perform much more poorly on the test. By designing rules that can account for such differences between groups and minimize its predictive error on the worst-off group, we can help account for such disparities and potentially dampen as opposed to amplify them.

\subsection{Summary of Contributions}
\begin{itemize}[leftmargin=*]
\item We extend the notion of minimax group fairness to strategic learning settings where a learner is interacting with a population of strategic agents, each of which belong to (only) one of $G$ groups. Given this notion, we formalize a fairness-aware Stackelberg game between the learner and the agents in which the agents manipulate their feature vector to maximize their utility which is measured by their classification outcome minus the cost of manipulation. In our model, we allow each group to have its own cost function.

Given a hypothesis class $\Hs$, the learner's objective in the game is one of the following: I) learn a minimax fair classifier in $\Hs$, II) learn a classifier in $\Hs$ that minimizes the overall error rate subject to minimax fairness constraints. We consider learning in the PAC model and work in the \emph{agnostic} setting in which the underlying ground-truth function that maps agents' feature vectors to labels may not necessarily belong to the hypothesis class of the learner $\Hs$.

\item When the cost functions of the agents are separable, we show the existence of a learning algorithm that solve both objectives of the learner \emph{even when $\Hs$ is the set of all classifiers}. This is because, as we show, the separability assumption makes the learning problem essentially $G$-dimensional \emph{regardless of how complex $\Hs$ is}. The algorithm that we propose for separable costs runs in time that is exponential in the number of groups $G$, so it is efficient only when $G$ is small.

\item We then consider cost functions that are not necessarily separable and ask for efficient reductions from fair learning to standard learning in the strategic setting. In other words, given an oracle that solves learning problems absent fairness constraints in the strategic setting, we ask whether the fair learning objectives of the learner can be solved by calling the oracle only polynomially many times, i.e., whether there exist \emph{oracle-efficient} algorithms. Here, we consider learning randomized classifiers (distributions over $\Hs$) by extending the hypothesis class of the learner $\Hs$ to the probability simplex over $\Hs$.

While often in machine learning the use of randomized classifiers can help convexify (linearize) the objective of the learner, we show how the standard model of strategic classification leads to non-convex optimization problems for the learner \emph{even when randomized classifiers are used}. We then propose a fully \emph{transparent} model of strategic classification to circumvent this convexity issue. In this transparent model, the random classifier of the learner is drawn from the learner's distribution $p$ before the interaction between the learner and the agents occurs. In other words, the learner first draws its classifier $h$ from its distribution $p$ and then reveals $h$ to the agents.

Given the convexity of the learner's objectives, we use techniques from online learning and game theory to develop oracle-efficient learning algorithms for the learner \emph{when $\Hs$ has finite Strategic VC dimension}. The first algorithm that we propose solves the first objective of the learner by making only $O \left( \log G \right)$ calls to the learning oracle. Our second algorithm makes $O \left(G^2 \right)$ oracle calls and solves the second objective of the learner. We note that unlike our algorithm for the case of separable costs, the running time of these algorithms are polynomial in the number of groups. 

\item We conclude by conducting experiments on real data, evaluating both objectives described earlier: learning a minimax fair classifier, and learning a classifier that minimizes the overall error rate while satisfying our minimax fairness constraint. 
One implementation challenge we face in practice is selecting a heuristic to replace the learning oracle required by the theory. 
In our empirical studies, we use the $\ell_2$-distance as the agent cost function, with different groups having different manipulation budgets. We consider the set of linear classifiers as the hypothesis class $\mathcal{H}$, and employ a simple heuristic---``{\em shifting the optimal non-strategic linear classifier}''---as our learning oracle in the strategic setting.
Our results demonstrate that, even with this simple heuristic, our proposed algorithm outperforms both non-strategic learners and a na\"ive post-processing strategic learner with respect to the fairness objective.
\end{itemize}

\subsection{Related Work}
This work most closely aligns with the fields of strategic classification and algorithmic fairness in machine learning. We cover works that are closely related to this paper below.

\paragraph*{Strategic Classification.}
Strategic classification  was first formalized by \cite{bruckner2011stackelberg, hardt2016strategic}.
\cite{hardt2016strategic} is perhaps the most seminal work in the area of strategic classification. They provide the first computationally efficient algorithms to learn an approximately optimal classifier in strategic settings under the assumption that the agents' cost function is separable. Part of our work assumes separability of the costs and  builds directly upon the algorithm and setting of \cite{hardt2016strategic}. In particular, we extend their results from the single-group to the multiple-group setting in which the population is partitioned into several groups and that each group can have its own separable cost function. Furthermore, in our multiple-group setting, we consider objectives that constrain the learned classifier to satisfy the notion of strategic minimax group fairness that we define.

Following \cite{hardt2016strategic}, several works have studied learning in the presence of strategic agents in both online and PAC models. Some of these works include: \cite{dong2018strategic, chen2020learning, zhang2021incentive, sundaram2023pac, ahmadi2021strategic, ahmadi2023fundamental, lechner2023strategic, cohen2024learnability, shao2024strategic}. In this work, we consider a PAC learning model and use the notion of Strategic VC Dimension introduced by \cite{sundaram2023pac} to characterize learnability in the presence of strategic agents when the goal is learning classifiers that satisfy the notion of minimax fairness.

\paragraph*{Algorithmic Fairness.}
Our work uses the notion of minimax group fairness as discussed in \cite{diana2021minimax} and \cite{pmlr-v119-martinez20a} where a classifier is considered fair if it minimizes the maximum group error rate. In addition to machine learning, minimax solutions are a standard approach to achieving fairness in other domains such as scheduling, fair division, dimensionality reduction, clustering, and portfolio design \citep{hahne1991round, asadpour2007approximation, samadi2018price,tantipongpipat2019multi,ghadiri2021socially,abbasi2021fair,makarychev2021approximation,diana2021algorithms}. In terms of techniques, when we consider general costs that are not necessarily separable, we use the reductions approach to fair classification that was first introduced by \cite{agarwal2018reductions}. Using techniques from online learning and game theory, they show that learning with group fairness constraints can be reduced to standard empirical risk minimization without any constraints. We use the same high-level game theoretic approach to develop algorithms for learning fair models in the strategic setting.

\paragraph*{Fairness in Strategic Settings.}
There are several works that study the social aspects of strategic classification in settings where the population consist of groups that have different costs. These works generally consider accuracy as the primary objective of the learner and study the social effects of deploying accuracy maximizing classifiers in strategic settings. In our work, however, we explicitly work with learning objectives that ask for fairness in addition to maximizing the accuracy of the learner. Another key difference between our work and prior work relies on how fairness is viewed in strategic settings: similar to standard learning settings, we consider fairness with respect to the \emph{outcomes} received by the agents, whereas the majority of works in this domain define fairness with respect to the \emph{costs} that the agents have to incur in order to receive a positive classification. We briefly review the most relevant works in this domain:

\cite{milli2019social} introduces the notion of \emph{social burden of a classifier} in strategic settings, which is defined as the expected cost that the qualified agents have to incur in order to receive positive classification from the classifier. Their main result shows that in a population with two groups, a classifier that maximizes the strategic accuracy for the learner can cause disparate amount of social burden among population groups. In particular, they show that the more robust the classifier becomes to strategic manipulation, the larger the gap will be between the social burdens of the two groups. 

\cite{hu2019disparate} considers a strategic setting with an advantaged and a disadvantaged group such that the disadvantaged group's cost is always higher than the advantaged group's cost. They show that adopting classifiers that maximize the strategic accuracy for the learner can exacerbate the existing inequalities among groups by mistakenly accepting unqualified agents from the advantaged group and rejecting qualified agents from the disadvantaged group.

\cite{keswani2023addressing} focuses on the strategic manipulation costs of different groups and takes the social burden gap as a fairness metric in strategic classification tasks and develops a constrained optimization framework that aims to maximize accuracy such that the gap in social burden is bounded by a given threshold.

\cite{estornell2021unfairness} studies the effect of strategic manipulations on the fairness of classifiers that are not robust to strategic manipulations. In particular, they consider a baseline model that is trained to maximize the accuracy for the learner, and a fair model that maximizes the accuracy subject to fairness constraints. Both of these models are learned in the standard non-strategic setting. \cite{estornell2021unfairness} give conditions under which agents' strategic manipulations can cause the fair model to become less fair than the baseline model. \cite{braverman2020role} studies the role of noise in strategic classification and find that in some cases, noisier signals can improve both accuracy and fairness for the learner.

\section{Model and Preliminaries}

Each agent in our framework is represented by a tuple $(x,g,y)$ where $x \in \mathcal{X}$ is the feature vector, $g \in \mathcal{G} \triangleq \{ 1,2, \ldots, G\}$ is the protected group that the agent belongs to, and $y \in \mathcal{Y} \triangleq \{ 0,1\}$ is the binary label. We note that in our framework each agent belongs to only one of $G$ groups, i.e., groups are assumed to be disjoint. We assume there exists a distribution $D$ over the data domain $\mathcal{X} \times \mathcal{G} \times \mathcal{Y}$. We let $D_g$ denote the conditional distribution of $(x,y)$ conditioned on the group $g$. Formally, we have for every $E \subseteq \mathcal{X} \times \mathcal{Y}$, the probability of $E$ under $D_g$ is given by
\[
\Pr_{(x,y) \sim D_g} \left[ (x,y) \in E \right] = \Pr_{(x,g', y) \sim D} \left[ (x,y) \in E \, | g' = g \right]
\]

Agents in a strategic setting are equipped with a cost function that captures their cost of manipulation. In our model, we allow each group to have its own cost function. Formally, for every $g \in \mathcal{G}$, the cost function of group $g$ is given by: $c_g: \mathcal{X} \times \mathcal{X} \to \mathbb{R}_+$ where $c_g (x,z)$ is the cost of manipulating the feature vector from $x$ to $z$ for an agent who belongs to group $g$. We note that in our model the agents cannot change their group membership.

Let $\mathcal{H} \subseteq \mathcal{Y}^\mathcal{X}$ be the hypothesis class of the learner. Similar to standard strategic classification settings, we consider a \emph{Stackelberg} game between the learner who is the ``leader'' meaning that she plays her strategy first, and the agents who are the ``followers'' meaning that they respond to the strategy of the learner. The learner's goal is to publish a classifier $h \in \mathcal{H}$ that minimizes its loss which we define later on. On the other hand, each agent in the game best responds to $h$ by a manipulation that maximizes their utility which is measured by the difference of their classification outcome and their manipulation cost. Formally, for an agent $(x,g)$, the corresponding utility of manipulating to $z \in \mathcal{X}$ when facing a classifier $h$ is given by
\begin{equation}\label{eq:agentutility}
    u_{(x,g)} (z; h) \triangleq h(z) - c_g (x,z)
\end{equation}
We let $\br (x,g,h)$ denote a best response of an agent $(x, g)$ to the classifier $h$, i.e., a point $z$ that maximizes $u_{(x,g)} (z, h)$ where ties are broken arbitrarily. We note that given this utility function, the ``feasible manipulation region'' for an agent $(x,g)$ is the set $\{ z: c_g (x,z) < 1 \}$ because $h(z) \in \{ 0, 1 \}$.

We now turn our attention to the objective of the learner and define its loss function. First, we define the group and the overall \emph{strategic} error rates of a classifier $h$, which we denote by $\ell_g(h)$ and $\ell (h)$, respectively.
\begin{definition}[Strategic Error Rates]
Given a distribution $D$ with corresponding group conditionals $\{ D_g\}_g$, the overall error rate of a hypothesis $h$, $\ell (h)$, and its corresponding group error rate for the group $g$, $\ell_g (h)$, are defined as follows:
\begin{equation*}
     \ell (h) \triangleq \Pr_{(x,g,y) \sim D} \left[ h\left( \br (x,g,h) \right) \neq y \right] , \quad \ell_g(h) \triangleq \Pr_{(x,y) \sim D_g} \left[ h\left( \br (x,g,h) \right) \neq y \right]
\end{equation*}
\end{definition}

These error rates are simply the expected misclassification rates after the agents commit to their best response strategies. Next, we define the notion of fairness that we work with throughout the paper. This notion asks for minimizing the error of the worse off group, i.e., the group with maximal error rate.
\begin{definition}[Strategic Minimax Fairness] We say a classifier $h \in \mathcal{H}$ satisfies ``$\gamma$-minimax fairness" with respect to the distribution $D$ if it minimizes the maximum group error rate up to an additive factor of $\gamma$. In other words,
\[
\max_{g \in \mathcal{G}} \ell_g(h) \le \min_{h' \in \mathcal{H}} \max_{g \in \mathcal{G}} \ell_g(h') + \gamma
\]
\end{definition}

We are ready to formally define our fairness-aware strategic classification game. We note that our game allows the learner to have only \emph{incomplete} information about the distribution of the agents $D$. This is represented by a set of examples $S$ drawn from $D$. We consider two objectives in this game for the learner: one in which the learner just wants to find a minimax fair classifier; and another in which the learner finds a minimax fair classifier that further minimizes the overall strategic error rate.

\begin{definition}[The Fairness-aware Strategic Game] 
\label{defn:fairgame}The game, between the learner and the agents, proceeds as follows:
    \begin{enumerate}
        \item The learner, knowing the cost functions $\{ c_g \}_{g \in \mathcal{G}}$, and having access to $S = \{ (x_i, g_i, y_i)\}_{i=1}^n$ drawn $i.i.d.$ from $D$, publishes a classifier $h \in \mathcal{H}$.
        \item Every agent $(x,g)$ best responds to $h$ by moving to a point $\br(x,g,h)$ that maximizes their utility.
        \begin{equation*}
         \br(x,g,h) \in \argmax_{z \in \mathcal{X}} u_{(x,g)} (z; h)
        \end{equation*}
        \end{enumerate}
        Given a threshold $\gamma > 0$, the learner's goal in this game is one of the following:
        \begin{itemize}
        \item \textbf{Objective I}: Find a $\gamma$- minimax fair classifier: find $h \in \mathcal{H}$ such that
        \[
        \max_{g \in \mathcal{G}} \ell_g(h) \le \min_{h' \in \mathcal{H}} \max_{g \in \mathcal{G}} \ell_g(h') + \gamma
        \]
        \item \textbf{Objective II}: Among all $\gamma$-minimax fair classifiers, find one that minimizes the overall error rate:
        \begin{equation}\label{eq:optt}
        \min_{h \in \mathcal{H}} \left\{ \ell (h): \max_{g \in \mathcal{G}} \ell_g(h) \le \min_{h' \in \mathcal{H}} \max_{g \in \mathcal{G}} \ell_g(h') + \gamma \right\} \triangleq \text{OPT} \left( \mathcal{H}, \gamma \right)
        \end{equation}
        \end{itemize}
\end{definition}
We let $\text{OPT} \left( \mathcal{H}, \gamma \right)$ denote the optimal value of optimization problem~(\ref{eq:optt}). Throughout the paper, we use $\ell_g$ and $\ell$ for error rates computed with respect to the unknown distribution $D$, and $\hat{\ell}_g$ and $\hat{\ell}$ for (empirical) error rates computed with respect to the dataset $S$. We use $n$ for the size of the dataset $S$ and $n_g$ for the size of group $g$ in the data set $S$. Note that $n = \sum_g n_g $. We will use some concepts and results from learning theory and game theory which we briefly discuss below.

\subsection{Learning Theory Preliminaries}
In this section, we cover necessary definitions and tools from learning theory which are taken from the standard literature on learning theory (see, e.g., \citet{kearns1994introduction}). We start with the definition of VC dimension which is a notion that captures the complexity of a hypothesis class.
\begin{definition}[VC dimension]\label{def:vc}
Let $\Hs \subseteq \mathcal{Y}^\mathcal{X}$ be a hypothesis class. For any $S = \{ x_1, \ldots, x_n \} \subseteq \X$, define $\Hs (S) = \{ (h(x_1), \ldots, h(x_n)): h \in \Hs \}$. We say $\Hs$ shatters $S$, if $ \Hs(S) = \{ 0,1 \}^n$, i.e., if $\Hs (S)$ contains all possible labelings of the points in $S$. The Vapnik-Chervonenkis (VC) dimension of $\Hs$, denoted by $VC \left( \Hs \right)$, is the cardinality of the largest set of points in $\X$ that can be shattered by $\Hs$. In other words,
\[
VC \left(\Hs \right) = \max \{n: \exists S \in \X^n \text{ such that $S$ is shattered by $\Hs$} \}
\]
If $\Hs$ shatters arbitrarily large sets of points in $\X$, $VC \left(\Hs \right) = \infty$.
\end{definition}

Next, Sauer's lemma bounds the number of labelings a class $\Hs$ can induce on a dataset of size $n$.
\begin{lemma}[Sauer's Lemma]\label{lem:sauer}
    Let $S$ be a data set of size $n$ and let $VC \left(\mathcal{H} \right) = d < \infty$. Define
    \[
    \mathcal{H} ( S ) = \left\{ \left( h \left( x_1 \right), \ldots, h \left( x_n \right) \right): h \in \mathcal{H} \right\}
    \]
    We have that $| \mathcal{H} ( S ) | = O \left( n^{d} \right)$.
\end{lemma}

A key result in learning theory states that learning in classes with finite VC dimension is guaranteed to generalize:

\begin{theorem}[Generalization for VC Classes]\label{thm:vc-gen}
Let $D$ be a distribution over the domain $\X \times \mathcal{Y}$. Suppose $\Hs$ is a hypothesis class with VC dimension $VC \left(\Hs \right) = d$. We have that for every $\epsilon, \delta \ge 0$, with probability at least $1-\delta$ over the $i.i.d.$ draws of $S \sim D^n$,
\[
\sup_{h \in \Hs} \left\vert \Pr_{(x,y) \sim D} \left[  h(x) \neq y \right] - \Pr_{(x, y) \sim S} \left[ h(x) \neq y \right] \right\vert \le \epsilon
\]
provided that
\[
n = \Omega \left( \frac{  d \log \left(n\right) + \log \left( 1 / \delta \right)  }{\epsilon^2} \right)
\]
\end{theorem}

Note that because we are in a strategic setting where agents can modify their feature vector in response to the published classifier $h$, these standard tools from learning theory do not directly apply to our framework. We need a complexity measure of the class $\Hs$ that takes into account the best response of the agents. \citet{sundaram2023pac} defines ``Strategic VC Dimension'' of a hypothesis which is the notion of complexity that we work with in this paper.

\begin{definition}[Strategic VC Dimension \citep{sundaram2023pac}]\label{def:svc}
Let $\mathcal{H}$ be a concept class and $\{ c_g \}_g $ be the group cost functions. Define $\mathcal{F} = \{ f_h: h \in \mathcal{H }\}$ as follows: $f_h: \mathcal{X} \times \mathcal{G} \to \mathcal{Y}, \, f_h (x, g) = h\left( \br (x,g,h) \right)$. The strategic VC dimension of the class $\mathcal{H}$ with respect to the costs $\{ c_g \}_g$, $SVC \left(\mathcal{H}, \{ c_g \}_g \right)$, is defined as the VC dimension of $\mathcal{F}$:
\[
SVC \left(\mathcal{H}, \{ c_g \}_g \right) = VC \left( \mathcal{F} \right)
\]
\end{definition}
Whenever it is clear from context, we drop the dependency on the costs $\{ c_g \}_g $ and simply write $SVC \left(\mathcal{H} \right)$ for the strategic VC dimension of the class $\Hs$. \citet{sundaram2023pac} shows that $SVC \left(\mathcal{H} \right)$ characterizes the learnability of class $\Hs$ in the strategic setting. As an example, when $\X = \reals^d$, and $\Hs$ is the set of linear classifiers in $\reals^d$, and the group cost functions are given by $c_g (x, z) = k_g \cdot \Vert x - z \Vert_p$ for some $p >1$ and $k_g > 0$, \citet{sundaram2023pac} shows that $SVC \left(\mathcal{H} \right) = d + 1$. For more details on the notion of Strategic VC dimension, see \citep{sundaram2023pac}.

\subsection{Game Theory Preliminaries}\label{subsec:noregret}

In this section, we briefly review the seminal result of \citet{fs1996} known as the ``No-regret Dynamics''. Consider a zero-sum game with two players: a player with strategies in $S_1$ (the minimization player) and another player with strategies in $S_2$ (the maximization player). Let $U: S_1 \times S_2 \to \mathbb{R}_{+}$ be the payoff function of this game. For every strategy $s_1 \in S_1$ of the minimization player and every strategy $s_2 \in S_2$ of the maximization player, the first player gets utility $-U(s_1, s_2)$ and the second player gets utility $U(s_1, s_2)$.
\begin{definition}[Approximate Equilibrium]\label{def:nuapprox}
A pair of strategies $(s_1, s_2) \in S_1 \times S_2$ is said to be a $\nu$-approximate equilibrium of the game if the following conditions hold:
\[
 U(s_1, s_2) - \min_{s'_1 \in S_1}  U(s'_1, s_2)  \le \nu,
\quad
\max_{s'_2 \in S_2}  U(s_1, s'_2) -  U(s_1, s_2)  \le \nu
\]
\end{definition}

In other words, $(s_1, s_2)$ is a $\nu$-approximate equilibrium of the game if neither player can gain more than $\nu$ by deviating from their strategies.

\citet{fs1996} proposed an efficient framework for finding an approximate equilibrium of the game: In an iterative fashion, have one of the players update their strategies using a no-regret learning algorithm, and let the other player best respond to the play of the first player. Then, the empirical average of each player's actions over a sufficiently long sequence of such play will form an approximate equilibrium of the game. The formal statement is given below.

\begin{theorem}[No-Regret Dynamics \citep{fs1996}]\label{thm:noregret}
    Let $S_1$ and $S_2$ be convex, and suppose the utility function $U$ is convex-concave: $U(\cdot, s_2): S_1 \to \mathbb{R}_{\ge 0}$ is convex for all $s_2 \in S_2$, and $U(s_1, \cdot): S_2 \to \mathbb{R}_{\ge 0}$ is concave for all $s_1 \in S_1$. Let $(s_1^1, s_1^2, \ldots, s_1^T)$ be the sequence of  play for the first player, and let $(s_2^1, s_2^2, \ldots, s_2^T)$ be the sequence of play for the second player. Suppose for $\nu_1,\nu_2 \ge 0$, the regret of the players jointly satisfies
    \[
    \sum_{t=1}^T U(s_1^t, s_2^t) - \min_{s_1 \in S_1} \sum_{t=1}^T U(s_1, s_2^t) \le \nu_1 T,
    \quad
    \max_{s_2 \in S_2} \sum_{t=1}^T U(s_1^t, s_2) - \sum_{t=1}^T U(s_1^t, s_2^t) \le \nu_2 T
    \]
    Let $\bar{s}_1 = \frac{1}{T}\sum_{t=1}^T s_1^t \in S_1$ and $\bar{s}_2 = \frac{1}{T}\sum_{t=1}^T s_2^t \in S_2$ be the empirical average play of the players. We have that the pair $(\bar{s}_1, \bar{s}_2)$ is a $(\nu_1+\nu_2)$-approximate equilibrium of the game. 
\end{theorem}

\section{Separable Costs: An Efficient Learner for Small G}\label{sec:sep}
In this section, we focus on developing algorithms that solve the objectives of the learner cast in Definition~\ref{defn:fairgame} when the group cost functions $\{ c_g \}_g$ are separable. We note that our results and techniques in this section can be seen as an extension of those in \citep{hardt2016strategic} from the single group setting to the multiple group setting. \citet{hardt2016strategic} shows that separability of the cost function in the single group setting makes the learning problem 1-dimensional, regardless of what $\Hs$ is. We show a natural extension: under certain conditions that we specify, the separability of the cost functions in the setting with $G$ groups makes the learning problem $G$-dimensional, regardless of how complex $\Hs$ is.

We start by giving the definition of separable costs. Suppose for every group $g$, there exists real-valued functions $a_g, b_g: \mathcal{X} \to \mathbb{R}$ with $a_g (\mathcal{X}) \subseteq b_g (\mathcal{X})$ such that the cost function $c_g$ can be written as
\begin{equation}
    c_g(x,z) = \max \left( b_g (z) - a_g (x), 0\right)
\end{equation}
 This condition is referred to as ``separability'' by \cite{hardt2016strategic}. Note that $a_g (\mathcal{X}) \subseteq b_g (\mathcal{X})$ guarantees that every agent has a manipulation with zero cost. Our results in this section hold for a family of separable cost functions $\{ c_g \}_g$, defined as above, that further satisfy the following conditions.

\begin{assumption}\label{ass:sep}
    We assume in this section that the separable cost functions $\{ c_g\}_g$ that are expressed by $\{(a_g,b_g)\}_g$ satisfy the following conditions:
    \[
    1. \, \forall h \in \mathcal{H}: \, \bigcap_{g' \in \mathcal{G}} \argmin_{z : h(z) = 1}  b_{g'} (z) \neq \emptyset, \quad 2. \, \forall (x,g) \in \mathcal{X} \times \mathcal{G}: \bigcap_{g' \in \mathcal{G}} \argmax_{z: c_g (x,z) < 1}  b_{g'} (z) \neq \emptyset
    \]
\end{assumption}
Condition 1 asks that for every classifier $h$, there exists a point in the positive region of $h$ that minimizes $b_{g'}$ \emph{simultaneously} for all $g'$. Condition 2 asks that for every agent $(x,g)$, there exists a point within cost $\le 1$ of the agent (the feasible region of manipulation for the agent) that maximizes $b_{g'} (z)$ \emph{simultaneously} for all $g'$. These conditions can be satisfied in natural cases which we discuss below.

Note that when there is only one group ($G=1$), which is the setting considered in \citet{hardt2016strategic}, or when all groups have the same cost function (for some $c$, $c_g = c, \, \forall g$), the assumption trivially holds. But are there separable cost functions that satisfy this assumption when we have at least two groups with different cost functions? We give one natural example and a sufficient condition for Assumption~\ref{ass:sep} below:
\begin{itemize}
    \item An example satisfying Assumption~\ref{ass:sep}: given $k_g \in \mathbb{R}^+$ for all $g$, and functions $a, b : \mathcal{X} \to \mathbb{R}$ with $a (\mathcal{X}) \subseteq b (\mathcal{X})$, the following cost functions are separable and satisfy the conditions of Assumption~\ref{ass:sep}.
    \[
    \forall g \in \mathcal{G}: \ c_g (x,z) = k_g \cdot \max \left(  b (z) -  a (x), 0 \right)\footnote{More generally, the assumption holds for the class $c_g (x,z) = \max \left( \alpha_g b (z) - \beta_g a (x), 0 \right)$ where $\alpha_g,\beta_g \in \mathbb{R}$.}
    \]
    This is a natural example where group cost functions differ by their respective $k_g$ value: some groups have to pay a higher cost for manipulation (large $k_g$) while others incur lower cost for the same manipulation (small $k_g$).
    \item A sufficient condition for Assumption~\ref{ass:sep}: for given separable cost functions $\{ c_g \equiv (a_g, b_g) \}_g$ define $\phi : \mathcal{X} \to \mathbb{R}^G$ by $\phi(x) = (b_g (x))_{g \in \mathcal{G}}$. Note that $\phi (\mathcal{X}) \subseteq \prod_{g \in \mathcal{G}} b_g (\mathcal{X})$ (Cartesian product of $b_g (\mathcal{X})$ sets); however, if $\phi (\mathcal{X}) = \prod_{g \in \mathcal{G}} b_g (\mathcal{X})$, then the cost functions satisfy the conditions of Assumption~\ref{ass:sep}.
\end{itemize}

We move on to develop an algorithm that finds an (approximately) optimal classifier for the learner when the group cost functions are separable and satisfy Assumption~\ref{ass:sep}. In this section, we show that the learner can take its hypothesis class $\mathcal{H}$ to be the set of all classifiers: $\mathcal{H} = \mathcal{Y}^\mathcal{X}$. To do this, we show that there exists a function class $\mathcal{F}$ with bounded complexity such that for separable cost functions that satisfy Assumption~\ref{ass:sep}, the optimal value of the game for the learner is the same when they optimize over only $\mathcal{F}$ instead of $\mathcal{H}$. Therefore, because $\mathcal{F}$ is sufficient for the purpose of finding the optimal classifier, this allows us to take $\mathcal{H}$ to be the set of all classifiers $\mathcal{H} = \mathcal{Y}^\mathcal{X}$. To formalize, given separable costs $\{ c_g \equiv (a_g, b_g) \}_g$, let $\mathcal{F} = \{ f_{t} : \mathcal{X} \to \mathcal{Y}: t = (t_1, \ldots, t_G) \in \mathbb{R}^G \}$ where
\[
f_{t} (x) \triangleq \prod_{g \in \mathcal{G} }\mathds{1} \left[ b_g (x) \ge t_g \right]
\]

\begin{lemma}[Sufficiency of Optimizing over $\mathcal{F}$]\label{lem:F}
    Let $\mathcal{H} = \mathcal{Y}^\mathcal{X}$ be the set of all classifiers. We have
    \begin{itemize}
    \item Objective I: $
    \min_{f \in \mathcal{F}} \max_g \ell_g (f) \le \min_{h \in \mathcal{H}} \max_g \ell_g (h)
    $.
    \item Objective II: for every $\gamma \ge 0$, $\text{OPT} \left(\mathcal{F}, \gamma \right) \le \text{OPT} \left(\mathcal{H}, \gamma \right)$.
    \end{itemize}
    
\end{lemma}

\begin{proof}[Proof of Lemma~\ref{lem:F}]
    For any classifier $h$, let $\Gamma(h) \triangleq \{ (x,g): h\left( \br (x,g,h) \right) = 1\}$ be the set of all agents that are classified as positive by the classifier, i.e., the acceptance region of $h$. Note that given the utility function of the agents (Equation~(\ref{eq:agentutility})), we can re-write $\Gamma(h)$ as
    \[
    \Gamma(h) = \left\{ (x,g) : a_g (x) > \min_{z: h(z) = 1} b_g (z) - 1 \right\}
    \]
    Fix any $h \in \mathcal{H}$ and define
    \[
    f: \mathcal{X} \to \mathcal{Y}, \quad f (x) = \prod_{g \in \mathcal{G}} \mathds{1} \left[ b_g (x) \ge \min_{z: h(z) = 1} b_g (z) \right]
    \]
    First, note that $f \in \mathcal{F}$ by construction. Furthermore, we have by the first part of Assumption~\ref{ass:sep} and the construction of $f$ that
    \[
    \min_{z: h(z) = 1} b_g (z) = \min_{z: f(z) = 1} b_g (z)
    \]
    Therefore, $\Gamma (h) = \Gamma (f)$. This completes the proof of the first part because for every $h \in \mathcal{H}$ we can find a $f \in \mathcal{F}$ such that $f$ induces the same labeling of the agents as $h$.

    To prove the second part, note that because of $\mathcal{F} \subseteq \mathcal{H}$ and the first part of the lemma,
    \[
    \min_{f' \in \mathcal{F}} \max_g \ell_g (f) = \min_{h' \in \mathcal{H}} \max_g \ell (h)
    \]
    The proof is complete by the same observation that for every $h \in \mathcal{H}$ we can find a $f \in \mathcal{F}$ such that $f$ induces the same labeling of the agents as $h$.
\end{proof}

We now describe the learning algorithm that we develop for the case of separable costs. Our plan is to develop an algorithm for solving the empirical version of the learner's problem in which error rates are computed with respect to a given dataset $S$ sampled $i.i.d.$ form the underlying distribution $D$. To complement our empirical guarantees, we prove uniform convergence theorems establishing that the optimal solution of the empirical problem remains (approximately) optimal with respect to the distribution so long as the dataset size is large enough (polynomial in the relevant parameters).

To solve the empirical problem, we first note by Lemma~\ref{lem:F} that, it suffices for us to optimize over classifiers in $\mathcal{F}$. But note that every classifier in $\mathcal{F}$ is basically a $G$-dimensional threshold function on $b_g (\cdot)$ functions: $
f_{t} (x) = \prod_{g \in \mathcal{G} }\mathds{1} \left[ b_g (x) \ge t_g \right]
$. Therefore, for a given $f_t$, an agent $(x,g)$ manipulates if and only if there exists a point $z$ in its feasible manipulation region $\{ z: c_g (x,z) < 1 \}$ such that $b_g (z) \ge t_g$ for all $g$. Given such structure, we can first compute the maximum $b_g (\cdot)$ values that the agents in $S$ can take within their feasible manipulation region. For data point $i$ in $S$, these values are denoted by $t_g^i$ in Algorithm~\ref{alg:separable}. Therefore for a given $f_t$, agent $i$ manipulates to receive a positive outcome if and only if $t_g^i \ge t_g$ for all $g$. Here, we use the fact that by Assumption~\ref{ass:sep}, there does exist a point $z^\star \in \mathcal{X}$ that simultaneously maximizes all $b_g$ functions for the agent. Therefore, agent $i$ is misclassified by $f_t$ if and only if $ y_i \neq \prod_{g} \mathds{1} \left[ t_{g}^i \ge t_{g}\right]$. Given these individual $t_g^i$ values that we compute in the algorithm, we construct a finite set of thresholds $T(S) \subseteq \reals^G$ that is guaranteed to contain a threshold which is optimal for the empirical problem, as we will prove. The algorithm will then enumerate over all thresholds in $T(S)$ to find one that optimizes the objective of the learner. We give a full description of our algorithm for separable costs in Algorithm~\ref{alg:separable} and provide its theoretical guarantees in Theorem~\ref{thm:sep} (objective I) and Theorem~\ref{thm:sep2} (objective II). Proofs of these theorems can be found in Appendix~\ref{app:sep}.

\begin{algorithm}[t]
\SetAlgoNoLine
    \KwIn{Dataset $S = \{ (x_i, g_i, y_i)\}_{i=1}^n$, separable costs $\{ c_g \equiv (a_g, b_g) \}_g$, desired fairness and error parameters $\gamma, \epsilon$.}
        For all $i \in [n]$ and $g \in [G]$, compute
        \begin{equation*}
            t_g^i  = \max_{z: c_{g_i} (x_i, z) < 1} b_g (z)
        \end{equation*}
        
        Let $T_g (S) = \{t_g^i\}_{i=1}^n \cup \{ \infty \}$ for all $g \in [G]$\;
        Construct the set of all possible thresholds: $T (S) = \prod_{g} T_g(S)$ \tcp*{Cartesian product; $\left| T(S) \right| = (n+1)^G$}
        For all $t = (t_1, \ldots, t_G) \in T(S)$, compute the group and overall empirical error rates of $f_t$:
        \begin{equation*}
            \forall g: \ \hat{\ell}_g(f_t) = \frac{\sum_{i=1}^n \mathds{1} \left[ g_i = g\right] \cdot \mathds{1} \left[ y_i \neq \prod_{g'} \mathds{1} \left[ t_{g'}^i \ge t_{g'}\right] \right]}{\sum_{i=1}^n \mathds{1} \left[ g_i = g\right]}, \quad \hat{\ell}(f_t) = \frac{\sum_{i=1}^n \mathds{1} \left[ y_i \neq \prod_{g'} \mathds{1} \left[ t_{g'}^i \ge t_{g'}\right] \right]}{n}
        \end{equation*}
        
    \KwOut{\textbf{Objective I}: $f_{\hat{t}}$ where $\hat{t} \in \argmin_{t \in T(S)} \max_g \hat{\ell}_g (f_t)$}
    \KwOut{\textbf{Objective II}: $f_{\hat{t}}$ where $\hat{t} \in \argmin_{t \in T(S)} \left\{ \hat{\ell} (f_t): \max_g \hat{\ell}_g(f_t) \le \min_{t \in T(S)} \max_g \hat{\ell}_g(f_t) +  \gamma + \epsilon \right\}$}
\caption{Minimax Fair Strategic Learning: Separable Costs, Objective I and II}
\label{alg:separable}
\end{algorithm}

\begin{restatable}{theorem}{objIsep}[Algorithm~\ref{alg:separable}: Objective I Guarantees]\label{thm:sep}
    Let $\mathcal{H} = \mathcal{Y}^\mathcal{X}$ be the set of all classifiers. There exists an algorithm (Algorithm~\ref{alg:separable}) such that for any data distribution $D$, any set of separable costs $\{ c_g \}_g$ satisfying Assumption~\ref{ass:sep}, and any $\gamma \ge 0$, runs in $O (n^G)$ time and for any $\delta$, with probability at least $1-\delta$ over the $i.i.d.$ draws of $S \sim D^n$, outputs a $\gamma$-minimax fair classifier $f_{\hat{t}} \in \mathcal{F}$: $\max_g \ell_g (f_{\hat{t}}) \le \min_{h \in \mathcal{H}} \max_g \ell_g (h) + \gamma$, provided that the size of the smallest group satisfies
    \[
    \min_{g \in \mathcal{G}} n_g = \Omega \left( \frac{G \log (n) + \log \left( G / \delta \right)}{\gamma^2} \right)
    \]
\end{restatable}

\begin{restatable}{theorem}{objIIsep}[Algorithm~\ref{alg:separable}: Objective II Guarantees]\label{thm:sep2}
    Let $\mathcal{H} = \mathcal{Y}^\mathcal{X}$ be the set of all classifiers. There exists an algorithm (Algorithm~\ref{alg:separable}) such that for any data distribution $D$, any set of separable costs $\{ c_g \}_g$ satisfying Assumption~\ref{ass:sep}, and any $\gamma, \epsilon$, runs in $O (n^G)$ time and for any $\delta$, with probability at least $1-\delta$ over the $i.i.d.$ draws of $S \sim D^n$, outputs a classifier $f_{\hat{t}} \in \mathcal{F}$ such that
    \begin{itemize}
        \item Fairness: $f_{\hat{t}}$ satisfies $(\gamma + 2\epsilon)$-minimax fairness: $\max_g \ell_g (f_{\hat{t}}) \le \min_{h \in \mathcal{H}} \max_g \ell_g (h) + \gamma + 2\epsilon$, and
        \item Accuracy: $f_{\hat{t}}$ satisfies $\ell (f_{\hat{t}}) \le \text{OPT} \left( \mathcal{H}, \gamma \right) + \epsilon$.
    \end{itemize}
    provided that the size of the smallest group satisfies
    \[
    \min_{g \in \mathcal{G}} n_g = \Omega \left( \frac{G \log (n) + \log \left( G / \delta \right)}{\epsilon^2} \right)
    \]
\end{restatable}

\section{General Cost Functions: An Oracle-efficient Learner}\label{sec:gen}
In this section, we consider cost functions that are not necessarily separable and show that there exists oracle-efficient\footnote{an algorithm that calls a given oracle only polynomially many times.} algorithms that solve the optimization problems of the learner if the hypothesis class $\Hs$ has finite Strategic VC dimension. We begin by expanding the class $\Hs$ to the probability simplex over $\Hs$ which we denote by $\Delta (\mathcal{H})$.
\[
\Delta (\mathcal{H}) = \left\{ p: p \text{ is a distribution over } \Hs \right\}
\]

Each $p \in \Delta (\mathcal{H})$ can be seen as a randomized classifier. Note that often times in learning problems, considering distributions over $\Hs$ help linearize the objective of the learner because the loss of a distribution can be expressed as the expected loss of a classifier drawn from the distribution. However, this is not necessarily true in the standard strategic setting.

\paragraph{The Challenge of Convexity.} Note that in the standard strategic setting, agents observe the model of the learner, which is a distribution $p \in \Delta (\mathcal{H})$ in our case. Only after the agents best respond to $p$ and commit to their manipulation, the learner draws a classifier $h \sim p$ and classify the agents according to $h$. With such order of operations, the loss of a distribution $p$ for the learner would be:

\[
\mathbb{E}_{h \sim p } \left[ \Pr_{(x,g,y) \sim D} \left[ h\left( \br (x,g, {\color{red}p} ) \right) \neq y \right] \right]
\]
where we note that because the best response of the agents depend on $p$, this is not a linear function in $p$ and in general is non-convex. We therefore consider a fully \emph{transparent} model of information release where the order of operations lead to linearity of the loss for the learner. In this model, the learner first draws its classifier $h$ from the distribution $p$ and then releases $h$ to the agents. The key observation here is that when the agents best respond to the classifier $h$ drawn from $p$, the loss of the learner would be
\[
\mathbb{E}_{h \sim p } \left[ \Pr_{(x,g,y) \sim D} \left[ h\left( \br (x,g, {\color{blue}h} ) \right) \neq y \right] \right]
\]
which is now a \emph{convex (linear)} function in $p$. Given this transparent model of information release, we define the overall and the group error rates for randomized models as the expected error rate when the classifier is drawn from the given distribution and the agents best respond to the drawn classifier. These quantities are formally defined in Definition~\ref{def:trans}. We note that such transparent model, where agents best respond to the deterministic hypothesis drawn from a randomized model, is also considered in \citet{ahmadi2023fundamental}: they show that in the online setting of strategic classification, realizing the randomness of randomized classifiers before the agents respond can help the the utility of both the learner and the agents.

\begin{definition}[Strategic Error Rates of Randomized Classifiers in the Transparent Model]\label{def:trans}
Given a distribution $D$ with corresponding group conditionals $\{ D_g\}_g$, the overall error rate of a distribution $p \in \Delta (\mathcal{H})$, $\ell (p)$, and its corresponding group error rate for the group $g$, $\ell_g (p)$, are defined as follows:
\begin{equation*}
    \ell (p) \triangleq \mathbb{E}_{h \sim p } \left[ \Pr_{(x,g,y) \sim D} \left[ h\left( \br (x,g,h) \right) \neq y \right] \right] \equiv \mathbb{E}_{h \sim p } \left[ \ell (h) \right]
\end{equation*}
\begin{equation*}
    \ell_g(p) \triangleq \mathbb{E}_{h \sim p } \left[ \Pr_{(x,y) \sim D_g} \left[ h\left( \br (x,g,h) \right) \neq y \right] \right] \equiv \mathbb{E}_{h \sim p } \left[ \ell_g(h) \right]
\end{equation*}
\end{definition}

Given the definitions of the loss functions for the learner, we formalize the Stackelberg between the learner and the agents in the transparent setting of this section below.

\begin{definition}[The Transparent Fairness-aware Strategic Game] 
\label{defn:transfairgame}The game, between the learner and the agents, proceeds as follows:
    \begin{enumerate}
        \item The learner, knowing the cost functions $\{ c_g \}_{g \in \mathcal{G}}$, and having access to $S = \{ (x_i, g_i, y_i)\}_{i=1}^n$ drawn $i.i.d.$ from $D$, chooses a distribution $p \in \Delta ( \mathcal{H})$ and publishes a classifier $h \sim p$.
        \item Every agent $(x,g)$ best responds to $h$ by moving to a point $\br(x,g,h)$ that maximizes their utility.
        \begin{equation*}
         \br(x,g,h) \in \argmax_{z \in \mathcal{X}} u_{(x,g)} (z; h)
        \end{equation*}
        \end{enumerate}
        Given a threshold $\gamma > 0$, the learner's goal in this game is one of the following:
        \begin{itemize}
        \item \textbf{Objective I}: Find a $\gamma$- minimax fair model: find $p \in \Delta (\mathcal{H})$ such that
        \[
        \max_{g \in \mathcal{G}} \ell_g(p) \le \min_{p' \in \Delta (\mathcal{H})} \max_{g \in \mathcal{G}} \ell_g(p') + \gamma
        \]
        \item \textbf{Objective II}: Among all $\gamma$-minimax fair models, find one that minimizes the overall error rate:
        \begin{equation}\label{eq:optdelta}
        \min_{p \in \Delta (\mathcal{H})} \left\{ \ell (p): \max_{g \in \mathcal{G}} \ell_g(p) \le \min_{p' \in \Delta (\mathcal{H})} \max_{g \in \mathcal{G}} \ell_g(p') + \gamma \right\} \triangleq \text{OPT} \left( \Delta (\mathcal{H}), \gamma \right)
        \end{equation}
        \end{itemize}
\end{definition}

We let $\text{OPT} \left( \Delta (\mathcal{H}), \gamma \right)$ denote the optimal value of optimization problem~(\ref{eq:optdelta}). In this section, we assume the learner has access to an oracle that can solve strategic learning over $\Hs$ \emph{absent any fairness constraints}.

\begin{definition}[Oracle $\text{WERM}_\mathcal{H}$]\label{def:oracle}
    We assume there exists an oracle $\text{WERM}_\mathcal{H}$ (Weighted Empirical Risk Minimization over $\Hs$) that for any given dataset $S$ and any set of group weights $\{ w_g\}_g$ solves
    \[
    \text{WERM}_\mathcal{H} \left( S, \{ w_g\}_g \right) \in \argmin_{h \in \mathcal{H}}  \sum_g w_g \hat{\ell}_g (h)
    \] 
\end{definition}
 Given such an oracle, our goal is to show that the learner's problem can be solved by calling $\text{WERM}_\mathcal{H}$ only polynomially many times, from which the existence of oracle-efficient algorithms is implied. In other words, we show that learning with fairness constraints in the strategic setting can be reduced to learning in the strategic setting without any constraints.

 We use the following high-level plan to develop our algorithms in this section:
 \begin{enumerate}
     \item Given a dataset $S$, we start by writing the empirical versions of the optimization problems.
     \item Objective I of the learner is already written in a minmax form. However, for objective II, we use Lagrangian duality to re-write the constrianed optimization problem as a minmax problem. In both cases, we note that the minmax theorem \citep{sion1958general} holds.
     \item Next, we view the resulting minmax problems as two-player zero-sum games and use the \emph{no-regret dynamics} framework to find approximate equilibrium of the games. To elaborate, an approximate equilibrium is found by simulating an iterative game in which one player best responds and the other uses a no-regret learning algorithm to update its strategies. The empirical average of the players' strategies will then form an approximate equilibrium.
     \item We will then show that the approximate equilibrium guarantees can be converted to optimality guarantees of the original optimization problems. Therefore, the output of the no-regret dynamics framework serves as an approximately optimal solution to the original optimization problems.
     \item Finally, we prove uniform convergence guarantees establishing that with large enough sample size, the same model learned by our algorithm has optimality guarantees with respect to the underlying distribution $D$.
 \end{enumerate}

Missing proofs of this section can be found in Appendix~\ref{app:gen}.
\subsection{Objective I: Find a Minimax Fair Model}
In this section, we focus on solving the following problem: given a target $\gamma$, find $p \in \Delta (\mathcal{H})$ such that
\[
\max_{g \in \mathcal{G}} \ell_g(p) \le \min_{p' \in \Delta (\mathcal{H})} \max_{g \in \mathcal{G}} \ell_g(p') + \gamma
\]
Suppose we are given a data set $S = \{ (x_i, g_i, y_i)\}_{i=1}^n$ sampled $i.i.d.$ from the distribution $D$. We will develop an algorithm that solves the minmax problem with respect to the empirical distribution induced by $S$, and appeal to generalization guarantees to show that the learned model satisfies $\gamma$-minimax fairness with respect to the underlying distribution $D$. We can cast our empirical problem as the following minmax problem:
\begin{equation}\label{eq:emp-problem}
\min_{p \in \Delta (\mathcal{H})} \max_{g \in \mathcal{G}} \hat{\ell}_g(p) = \text{OPT}
\end{equation}
where we use the empirical loss function $\hat{\ell}_g$ instead of ${\ell}_g$. Define \[
\Lambda \triangleq \left\{ \lambda = (\lambda^1, \ldots, \lambda^G) \in \mathbb{R}_+^G: \Vert \lambda \Vert_1 \le 1 \right\}
\]
\[
\mathcal{H} ( S ) \triangleq \left\{ \left( h \left( \br \left(x_1, g_1, h \right) \right), \ldots, h \left( \br \left(x_n, g_n, h \right) \right) \right): h \in \mathcal{H} \right\}
\]
which is the set of all labelings induced by $\mathcal{H}$ on the data set $S$. When solving the empirical problem, note that it suffices for us to optimize over only $\Delta (\mathcal{H} (S))$. It follows directly from Sauer's Lemma (Lemma~\ref{lem:sauer}) and the definition of strategic VC dimension (Definition~\ref{def:svc}) that:
\begin{lemma}
    Let $S$ be a data set of size $n$ and let $SVC (\mathcal{H}) = d_{\mathcal{H}} < \infty$. We have that $| \mathcal{H} ( S ) | = O \left( n^{d_{\mathcal{H}}} \right)$.
\end{lemma}

Given $\Lambda$ and $\mathcal{H} ( S )$, note that we can re-write the minmax problem as:
\[
\min_{p \in \Delta (\mathcal{H} (S))} \max_{\lambda \in \Lambda} \left\{ \sum_g \lambda^g   \hat{\ell}_g(p)  \right\} = \text{OPT}
\]
Observe that both $\Delta (\mathcal{H} (S))$ and $\Lambda$ are convex and compact, and that the objective function of this minmax problem is linear in its variables $p$ and $\lambda$. As a consequence, the minmax theorem \citep{sion1958general} implies:
\[
\min_{p \in \Delta (\mathcal{H} (S))} \max_{\lambda \in \Lambda} \sum_g \lambda^g  \hat{\ell}_g(p) = \max_{\lambda \in \Lambda} \min_{p \in \Delta (\mathcal{H} (S))} \sum_g \lambda^g \hat{\ell}_g(p) = \text{OPT}
\]
We can now view the minmax optimization problem with the linear objective as a two-player zero-sum game: the primal player (the Learner) has strategies in $\Delta (\mathcal{H} (S))$ and wants to minimize the objective; the dual player has strategies in $\Lambda$ and wants to maximize the objective. We first show that a $\nu$-approximate equilibrium (Definition~\ref{def:nuapprox}) of this game corresponds to a $2 \nu$-approximate optimal solution for the learner's minmax problem~(\ref{eq:emp-problem}).

\begin{restatable}{lemma}{optminmax}[Equilibrium $\rightarrow$ Optimality]\label{lem:optimality}
Suppose $(\hat{p}, \hat{\lambda})$ is a $\nu$-approximate equilibrium of the game described above. We have that:
\[
 \max_{g \in \mathcal{G}} \hat{\ell}_g( \hat{p} ) \le \min_{p \in \Delta (\mathcal{H})} \max_{g \in \mathcal{G}} \hat{\ell}_g(p) + 2 \nu 
\]
\end{restatable}

We use the No-regret Dynamics (briefly discussed in Section~\ref{subsec:noregret}) to find an approximate equilibrium: in an iterative fashion, we let the learner best respond and the dual player use the Exponential Weights algorithm \citep{cesa1999prediction} which is a no-regret learning algorithm. The final output of the algorithm is the time-averaged plays of the two players which will converge to the equilibrium of the game by Theorem~\ref{thm:noregret}. We note that at every round of the game, given the strategy of the dual player $\lambda \in \Lambda$, the learner's best response strategy corresponds to solving
$
    \argmin_{h \in \mathcal{H}}  \sum_g \lambda^g \hat{\ell}_g (h)
$
which can be solved by invoking the oracle $\text{WERM}_\mathcal{H} \left(S, \lambda \right)$. This is because
\[
\min_{p \in \Delta (\mathcal{H})} \sum_g \lambda^g   \hat{\ell}_g(p) = \min_{p \in \Delta (\mathcal{H})} \mathbb{E}_{h \sim p} \left[ \sum_g \lambda^g   \hat{\ell}_g(h) \right] = \min_{h \in \Hs } \sum_g \lambda^g   \hat{\ell}_g(h)
\]
We present the algorithm of this section (for solving objective I of the learner) in Algorithm~\ref{alg:general1}. Guarantees of this algorithm are given in Theorem~\ref{thm:general1}.

\begin{algorithm}[t]
\SetAlgoNoLine
    \KwIn{Dataset $S = \{ (x_i, g_i, y_i)\}_{i=1}^n$, cost functions $\{ c_g \}_{g}$, desired fairness parameter $\gamma$.}
        Set the number of iterations and the learning rate: \[T  = \frac{8 \log G}{\gamma^2}, \ \eta = \sqrt{ \frac{8 \log G}{T} }\]
        
        Initialize the dual player's strategy: $\lambda_0 = (\frac{1}{G}, \ldots, \frac{1}{G}) \in \Lambda$\;
        \For{$t = 1, \ldots, T$}{
        Solve $h_t = \argmin_{h \in \mathcal{H}}  \sum_g \lambda_{t-1}^g \hat{\ell}_g (h)$ by calling the oracle $\text{WERM}_\mathcal{H} \left(S, \lambda_{t-1} \right)$ \tcp*{Best Response}
        Update: $\forall g, \, \lambda_t^g = \lambda_{t-1}^g \cdot \frac{1}{Z_t} e^{\eta \hat{\ell}_g (h_t)}$ where $Z_t = \sum_g e^{\eta \hat{\ell}_g (h_t)}$ \tcp*{Exponential Weights update}
        }
        \KwOut{$\hat{p} = \frac{1}{T} \sum_t h_t$: the uniform distribution over $\{ h_1, \ldots, h_T\}$}
\caption{Minimax Fair Strategic Learning: General Costs, Objective I}
\label{alg:general1}
\end{algorithm}

\begin{restatable}{theorem}{objIgen}[Guarantees of Algorithm~\ref{alg:general1}]\label{thm:general1}
    Fix a set of group cost functions $\{ c_g \}_{g}$. Suppose $\mathcal{H}$ has finite strategic VC dimension: $d_\mathcal{H} = SVC \left( \Hs \right) < \infty$. There exists an algorithm (Algorithm~\ref{alg:general1}) such that given access to the oracle $\text{WERM}_\mathcal{H}$, for any data distribution $D$, and any $\gamma \ge 0$, makes $O \left(\log G / {\gamma^2} \right)$ oracle calls, and for any $\delta$, with probability at least $1-\delta$ over the $i.i.d.$ draws of $S \sim D^n$, outputs a $\gamma$-minimax fair model $\hat{p} \in \Delta (\mathcal{H})$: $\max_g \ell_g (\hat{p}) \le \min_{p \in \Delta (\mathcal{H})} \max_g \ell_g (p) + \gamma$, provided that
    \[
   \min_{g \in \mathcal{G}} n_g = \Omega \left( \frac{ d_\mathcal{H} \log (n) + \log \left( G / \delta \right)}{\gamma^2} \right)
    \]
\end{restatable}

\subsection{Objective II: Optimize Error Subject to Minimax Fairness}

In this section, we focus on solving:
\begin{equation}\label{eq:II}
        \min_{p \in \Delta (\mathcal{H})} \left\{ \ell (p): \max_{g \in \mathcal{G}} \ell_g(p) \le \min_{p' \in \Delta (\mathcal{H})} \max_{g \in \mathcal{G}} \ell_g(p') + \gamma \right\} \triangleq \text{OPT} \left( \Delta (\mathcal{H}), \gamma \right)
\end{equation}
Similar to the previous section, given a data set $S$ sampled from the distribution, we solve the empirical optimization problem and appeal to generalization guarantees to show that if the size of the smallest group is large enough, the optimal solution of the empirical problem is also an approximately optimal solution for problem~(\ref{eq:II}). We can cast the empirical problem as:
\[
\min_{p \in \Delta (\mathcal{H})} \left\{ \hat{\ell} (p): \max_{g \in \mathcal{G}} \hat{\ell}_g(p) \le \min_{p' \in \Delta (\mathcal{H})} \max_{g \in \mathcal{G}} \hat{\ell}_g(p') + \gamma \right\}
\]
To solve this problem, we first call Algorithm~\ref{alg:general1} to estimate the minmax value
$
\min_{p' \in \Delta (\mathcal{H})} \max_{g \in \mathcal{G}} \hat{\ell}_g(p')
$.
Let $\tilde{p}$ be the output of Algorithm~\ref{alg:general1} and let $\hat{\gamma} \triangleq \max_g \hat{\ell}_g (\tilde{p})$ be our estimate for the minmax value. Given $\hat{\gamma}$, our problem can then be cast as a constrained optimization problem:
\begin{equation}\label{eq:opthat}
\min_{p \in \Delta (\mathcal{H})} \left\{ \hat{\ell} (p): \max_{g \in \mathcal{G}} \hat{\ell}_g(p) \le \hat{\gamma} + \gamma \right\} \triangleq \widehat{\text{OPT}} \left( \Delta (\mathcal{H}), \gamma \right)
\end{equation}
Similar to the previous section, we can restrict the set of variables to $\Delta (\mathcal{H} (S))$. Furthermore, by introducing a constraint for every group $g$, we can re-write the problem as
\[
\min_{p \in \Delta (\mathcal{H} (S))} \left\{ \hat{\ell} (p): \forall g, \  \hat{\ell}_g(p) \le \hat{\gamma} + \gamma \right\}
\]
Now for every constraint $g$, we introduce a dual variable $\lambda^g \in \reals_+$. The Lagrangian of the problem can then be written as:
\[
\mathcal{L} (p, \lambda) = \hat{\ell} (p) + \sum_g \lambda^g \left( \hat{\ell}_g(p) - \hat{\gamma} - \gamma\right)
\]
which is a linear function in both $p$ and $\lambda$. Applying the minmax theorem \citep{sion1958general}, we have that
\[
\min_{p \in \Delta (\mathcal{H} (S))} \max_{\lambda \in \mathbb{R}_+^G} \mathcal{L} (p, \lambda) = \max_{\lambda \in \mathbb{R}_+^G} \min_{p \in \Delta (\mathcal{H} (S))} \mathcal{L} (p, \lambda) =\widehat{\text{OPT}} \left( \Delta (\mathcal{H}), \gamma \right)
\]
Therefore, finding an optimal solution for the empirical problem~(\ref{eq:opthat}) reduces to solving the minmax problem with respect to the Lagrangian function. We follow a similar game-theoretic approach (no-regret dynamics) that we used in the previous section where we compute an approximately optimal minmax solution by simulating a two-player game between the primal player (the Learner) who controls $p$ and the dual player who controls $\lambda$. To guarantee convergence of the algorithm, however, we need to restrict the set of strategies for the dual player to the following \emph{bounded} set:
\[
\Lambda = \left\{ \lambda = (\lambda^1, \ldots, \lambda^G) \in \mathbb{R}_+^G: \Vert \lambda \Vert_1 \le B \right\}
\]
where $B$ will be chosen carefully in our algorithm. Note that because $\Lambda$ is convex and compact, the minmax theorem continues to hold:
\[
\min_{p \in \Delta (\mathcal{H} (S))} \max_{\lambda \in \Lambda} \mathcal{L} (p, \lambda) = \max_{\lambda \in \Lambda} \min_{p \in \Delta (\mathcal{H} (S))} \mathcal{L} (p, \lambda)
\]

In the next lemma, we show that a $\nu$-approximate equilibrium (Definition~\ref{def:nuapprox}) of the game can still have optimality guarantees for the learner even though we have restricted the space of the dual player.

\begin{restatable}{lemma}{opterror}[Equilibrium $\rightarrow$ Optimality]\label{lem:optimality2}
Suppose $(\hat{p}, \hat{\lambda})$ is a $\nu$-approximate equilibrium of the game described above. We have that:
\[
 \hat{\ell} (\hat{p}) \le \widehat{\text{OPT}} \left( \Delta (\mathcal{H}), \gamma \right) + 2 \nu
\]
and for any group $g$,
\[
\hat{\ell}_g (\hat{p}) \le \hat{\gamma} + \gamma + \frac{1 + 2 \nu}{B}
\]
\end{restatable}

Because of these optimality guarantees, if we find an approximate equilibrium of the game, the learner's strategy in that equilibrium will be an approximately optimal solution for the constrained optimization problem~(\ref{eq:opthat}). Similar to the previous section, we use the No-regret Dynamics to find an approximate equilibrium: In an iterative fashion, we let the learner best respond and the dual player use the Online Projected Gradient Descent algorithm \citep{zinkevich2003online} which is a no-regret learning algorithm. At round $t$, given $h_t$ of the learner, the gradient vector of the dual player is given by
\begin{equation}\label{eq:gradient}
\nabla_\lambda \mathcal{L} (h_t, \lambda ; \hat{\gamma} + \gamma) = \left[ \hat{\ell}_1(h_t) - \hat{\gamma} - \gamma, \, \hat{\ell}_2(h_t) - \hat{\gamma} - \gamma, \, \ldots, \hat{\ell}_G(h_t) - \hat{\gamma} - \gamma \right]^\top \in \mathbb{R}^G
\end{equation}
We note that at every round of the game, given the strategy of the dual player $\lambda \in \Lambda$, the primal player's best response strategy corresponds to solving:
\[
    \argmin_{p \in \Delta (\mathcal{H})}  \mathcal{L} (p, \lambda) = \argmin_{h \in \mathcal{H}}  \mathcal{L} (h, \lambda) = \hat{\ell} (h) + \sum_g \lambda^g \left( \hat{\ell}_g(h) - \hat{\gamma} - \gamma\right) 
\]
Note that $\hat{\ell} (h) = \sum_g (n_g / n) \cdot  \hat{\ell}_g(h)$. Therefore, this best response can be written as
\[
\argmin_{h \in \mathcal{H}} \sum_g \left( \lambda^g + \frac{n_g}{n} \right) \hat{\ell}_g(h)
\]
which can be solved with the learning oracle $\text{WERM}_\mathcal{H} \left(S, \{ \lambda^g + \frac{n_g}{n} \}_g \right)$. The final output of the algorithm is the time-averaged plays of the two players which will converge to the equilibrium of the game by Theorem~\ref{thm:noregret}. This algorithm is presented in Algorithm~\ref{alg:general2}. In this algorithm, $\text{Proj}_\Lambda (\cdot)$ is the $\ell_2$-projector operator that projects every point onto the set $\Lambda$. More precisely, for any $z \in \reals^G$,
\begin{equation}\label{eq:projop}
    \text{Proj}_\Lambda ( z ) \triangleq \argmin_{z' \in \Lambda} \Vert z' - z \Vert_2
\end{equation}

Theoretical guarantees of this algorithm are given in Theorem~\ref{thm:general2}.

\begin{algorithm}[t]
\SetAlgoNoLine
    \KwIn{Dataset $S = \{ (x_i, g_i, y_i)\}_{i=1}^n$, cost functions $\{ c_g \}_{g}$, desired fairness and error parameters $\gamma$ and $\epsilon$.}
        Set the number of iterations, the learning rate, and the upper bound on dual variables: \[T  = \left( \frac{4}{\epsilon} \left( \frac{8}{\epsilon^2} + G \right) \right)^2,  \ \eta_t = t^{-\frac{1}{2}}, \ B = \frac{4}{\epsilon}\]

        Run Algorithm~\ref{alg:general1} with input parameter $\epsilon$ to get $\tilde{p}$ and set: $\hat{\gamma} = \max_g \hat{\ell}_g (\tilde{p})$ \tcp*{Estimated minmax value}
        Initialize the dual player's strategy: $\lambda_0 = 0 \in \Lambda$\;
        \For{$t = 1, \ldots, T$}{
        Compute the group weights: $\forall g, \, w_g = \lambda_{t-1}^g + (n_g / n) $\;
        Solve $h_t = \argmin_{h \in \mathcal{H}}  \sum_g w_g \hat{\ell}_g(h)$ by calling the oracle $\text{WERM}_\mathcal{H} \left(S, \{ w_g \}_g \right)$ \tcp*{Best Response}
        Update: $\lambda_t = \text{Proj}_\Lambda \left( \lambda_{t-1} + \eta_t \nabla_\lambda \mathcal{L} \left(h_t, \lambda_{t-1} ; \hat{\gamma} + \gamma + \epsilon \right) \right) $ \tcp*{PGD update; see Eq.~(\ref{eq:gradient}) and (\ref{eq:projop})}
        }
        \KwOut{$\hat{p} = \frac{1}{T} \sum_t h_t$: the uniform distribution over $\{ h_1, \ldots, h_T\}$}
\caption{Minimax Fair Strategic Learning: General Costs, Objective II}
\label{alg:general2}
\end{algorithm}

\begin{restatable}{theorem}{objIIgen}[Guarantees of Algorithm~\ref{alg:general2}]\label{thm:general2}
    Fix a set of cost functions $\{ c_g \}_{g}$. Suppose $\mathcal{H}$ has finite strategic VC dimension: $d_\mathcal{H} = SVC \left( \Hs \right) < \infty$. There exists an algorithm (Algorithm~\ref{alg:general2}) such that given access to the oracle $\text{WERM}_\mathcal{H}$, for any data distribution $D$, and any $\epsilon, \gamma \ge 0$, makes $ O \left( G^2 \epsilon^{-6} \right)$ oracle calls, and for any $\delta$, with probability at least $1-\delta$ over the $i.i.d.$ draws of $S \sim D^n$, outputs a model $\hat{p} \in \Delta (\mathcal{H})$ such that
    \begin{itemize}
        \item Fairness: $\hat{p}$ satisfies $(\gamma + 3\epsilon)$-minimax fairness: $\max_g \ell_g (\hat{p}) \le \min_{p \in \Delta (\mathcal{H})} \max_g \ell_g (p) + \gamma + 3\epsilon$, and
        \item Accuracy: $\hat{p}$ satisfies $\ell (\hat{p}) \le {\text{OPT}} \left( \Delta (\mathcal{H}), \gamma \right) + \epsilon$.
    \end{itemize}
    provided that
    \[
   \min_{g \in \mathcal{G}} n_g = \Omega \left( \frac{ d_\mathcal{H} \log (n) + \log \left( G / \delta \right)}{\epsilon^2} \right)
    \]
\end{restatable}

\section{Experiments}
In this section, we provide empirical evaluations of our algorithms in section~\ref{sec:gen} for minimax optimization in a strategic environment, comparing with both a na\"ive strategic and a non-strategic baseline. Here is a brief summary of our empirical findings and contributions:
\begin{itemize}
    \item \textbf{Algorithm~\ref{alg:general1}}: While the na\"ive strategic baseline is roughly as effective as our strategic-aware algorithms when all agents have the same manipulation budget, the gap in performance between these two approaches increases significantly as the disparity in manipulation budgets among groups grows.
    \item \textbf{Algorithm~\ref{alg:general2}}: When considering the accuracy vs. fairness tradeoffs, our approach typically outperforms the other two, often Pareto-dominating both the non-strategic and strategic baseline. We found that with the upper bounds on maximal group error chosen, however, the Pareto curves for each approach were fairly limited in range.
\end{itemize}
\subsection{Experiment Methodology and Data}
\paragraph*{Datasets.} 
First we provide a brief description of the datasets we used in our experiments.
\begin{itemize}[leftmargin=*]
    \item {\bf Communities and Crimes.} This dataset contains socio-economic data from the 1990 US Census, law enforcement data from the 1990 US LEMAS survey, and crime data from the 1990 US FBI CUR~\citep{communities_and_crime_183}.
    \item {\bf COMPAS.} This dataset contains arrest data from Broward County, Florida, originally compiled by ProPublica~\citep{compas}.
    \item {\bf Credit.} This dataset, known as the German Credit Data, is used for predicting good and bad credit risks~\citep{statlog_(german_credit_data)_144}.
    \item {\bf Heart.} This dataset contains the medical records of 299 patients who experienced heart failure, collected during their follow-up period, with each patient profile containing 13 clinical features~\citep{heart_failure_clinical_records_519}.
\end{itemize}
Table~\ref{tab:datasets-info} summarizes the datasets used in our experiments. For each dataset, categorical features were transformed into one-hot encoded vectors, and group labels (which are based on sensitive attributes) were excluded from the usable features for the final models.
\begin{table}[!ht]
\centering
{\renewcommand{\arraystretch}{1}%
    \fontsize{6pt}{10pt}\selectfont
    \caption{Specifications of the datasets used in our experiments.}
    \begin{tabular}{l|c|c|c|c}
        {\bf Datasets} & {\bf Size ($n$)} & {\bf Dim. ($d$)} & {\bf Sensitive Attributes} 
        & {\bf Group Sizes} \\
        \midrule\midrule
        Communities & 1994 & 119 & race 
        & (1572, 219, 88, 115)\\
        \midrule
        COMPAS & 7164 & 5 & race 
        & (377, 3696, 2454, 637)\\
        \midrule
        Credit & 1000 & 19 & sex \& marital status 
        & (548, 310, 50, 92) \\
        \midrule
        Heart & 299 & 11 & sex &
        (194, 105) \\
    \end{tabular}
    \label{tab:datasets-info}
}
\end{table}

We remark that in all our experiments, similar to our theory, we assume that the dataset represents agents prior to any strategic manipulation; in other words, it captures their true feature vectors. This is a common assumption in strategic classification; for example, see~\cite{hardt2016strategic}.

\paragraph*{Hypothesis Class, Loss Function, and Oracle.}
We work with linear threshold functions, i.e., we take $\Hs$ to be the set of all classifiers that have the form: $h(x) = \mathds{1}[w^\top x + b \ge 0]$ for some $w \in \reals^d$ and $b \in \reals$. In all experiments, we consider the classification task with a 
$0/1$ loss function. In theory, our algorithms assume access to a weighted empirical risk minimization oracle (Definition~\ref{def:oracle}) that can efficiently find a model from our class $\Hs$ to minimize the learner's loss at each round. In practice, we do not have access to such an oracle and therefore need to employ a heuristic. As our heuristic at each round, we take the Paired Regression Classifier (PRC) Heuristic, used previously in non-strategic settings~\citep{gerrymandering,diana2021minimax,agarwal2018reductions}, and modify it to account for strategic manipulations of the agents. The modified PRC heuristic (\emph{Robust PRC}) that we work with takes as input a ``shift'' parameter $\rho$, runs the original PRC to get a classifier $h$, and shifts the learned linear classifier $h$ by $\rho$ towards the positive region of the classifier to make it robust to strategic manipulations. 
Formally, after shifting the classifier $h(x) = \mathds{1}[w^\top x + b \ge 0]$ by $\rho$, the resulting classifier is  $h'(x) \triangleq \mathds{1}[w^\top x + b' \geq 0]$, with $b' = b - \rho \|w\|_2$.
More information on this heuristic is provided in Appendix~\ref{app:exp}.

\paragraph*{Agent's Cost Function.}
We consider cost functions that are defined as follows: for every group $g$,
$
c_g (x, z) = \|x - z\|_2 / \tau_g
$.
Here, $\tau_g$ is a parameter which corresponds to the ``manipulation budget'' of members of group $g$: higher values of $\tau_g$ means that members of group $g$ have higher manipulation power.
Given the cost functions and a classifier $h$, if an agent $(x,g)$ satisfies $h(x) = 1$, then she does not manipulate. However, if $h(x) = 0$, she will manipulate if and only if there exists some $z$ such that $h(z) = 1$ and $\| x - z \|_2 \le \tau_g$. More precisely, when $h(x) = 0$, agent $(x, g)$ manipulates to $z = x + \tau_g \cdot \frac{w}{\|w\|_2}$ if and only if $h(z) = 1$.

\subsection{Baselines and our Method}
In our experiments, we divide the process into training (\texttt{Tr}) and testing (\texttt{Ts}) phases. We randomly select $30\%$ of the dataset as the test set, ensuring consistency by using the same random seed throughout. We compare our algorithms with the following baselines.

    \begin{itemize}[leftmargin=*]     
    \item{\bf Non-Strategic Learner:} This baseline completely ignores the strategic manipulations of agents during the \texttt{Tr} phase. Specifically, in the \texttt{Tr} phase, the learner assumes that agents are non-strategic and uses the minimax optimization algorithms of \cite{diana2021minimax} to learn a model $p \in \Delta(\Hs)$. We emphasize that during the learning of $p$ in the \texttt{Tr} phase, all error and weight updates are calculated under the assumption that agents do not manipulate. Then, in the \texttt{Ts} phase, the learned model $p$ is deployed without any modifications, and its error is computed with respect to \emph{strategic} agents, who manipulate based on their assigned manipulation budget.
    
    \item{\bf Na\"ive Strategic Learner:} This baseline can be seen as a simple {\em post-processing} approach to account for the strategic behavior of agents. Similar to the previous baseline, the same model $p$ is learned using the algorithms of \cite{diana2021minimax} during the \texttt{Tr} phase. However, during the \texttt{Ts} phase, the learner adjusts every classifier in the support of the distribution $p$ (uniform over some $\{h_1, \ldots, h_T\}$) by shifting each $h_t$ by $\tau_{avg}$ units towards the positive region of $h_t$, where $\tau_{avg}$ represents the average manipulation budget across all individuals in the data.   
\end{itemize}

In contrast, our approach can be seen as an {\em in-processing} method that accounts for the strategic manipulations of agents in the \texttt{Tr} phase. More precisely, it is defined as follows:
\begin{itemize}[leftmargin=*]
    \item{\bf Ours:} In the \texttt{Tr} phase, we deploy Algorithms~\ref{alg:general1} and \ref{alg:general2}, equipped with the Robust PRC heuristic for the learning oracle to learn model $p_o \in \Delta(\Hs)$. When all groups have the same manipulation budgets, $\forall g: \, \tau_g = \tau$, our algorithms choose $\rho = \tau$ for the shift parameter of the heuristic. However, when groups have different manipulation budgets, in each round of the algorithms, our algorithms perform a grid search in $[0, \max_{g\in\mathcal{G}} \tau_g]$ to choose the shift parameter $\rho$ that gives the lowest maximum group error rate. Then, in the \texttt{Ts} phase, $p_o$ is deployed, and the error is computed with respect to strategic agents, who manipulate based on their assigned manipulation budget.
\end{itemize}

\subsection{Equal Manipulation Budgets for All Agents}
In this section and Section~\ref{sec:diff-group-budget}, we examine the performance of Algorithm~\ref{alg:general1}, which computes a minimax fair model, and compare it to the two previously described baselines with respect to the minimax objective.

First, we consider the setting where all agents have the same budget $\tau$. In particular, we examine the performance of the three described methods across all four datasets for values of $0 \le \tau \le 6$ in increments of $0.5$. The performance plots, showing the min-max objectives of the three methods, are presented in Fig.~\ref{fig:equal-budget}.

\begin{figure*}[!ht]
    \centering
    \subfigure[]{
        \includegraphics[width=0.32\textwidth]{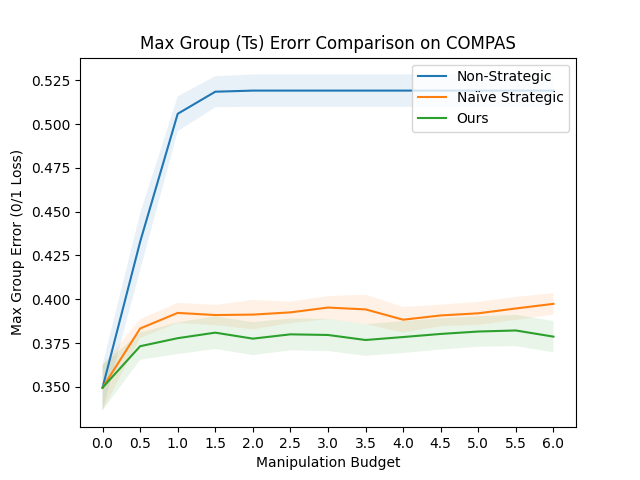}
        \label{fig:COMPAS_equal}
    }
    \quad \quad
    \subfigure[]{
        \includegraphics[width=0.32\textwidth]{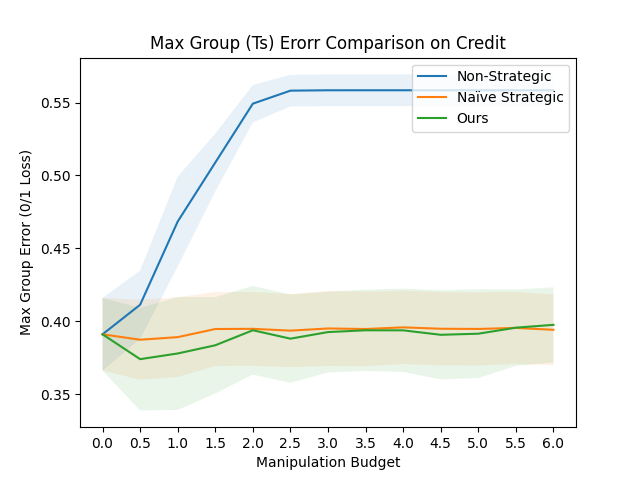}
        \label{fig:credit_equal}
    }
    \quad \quad
    \subfigure[]{
        \includegraphics[width=0.32\textwidth]{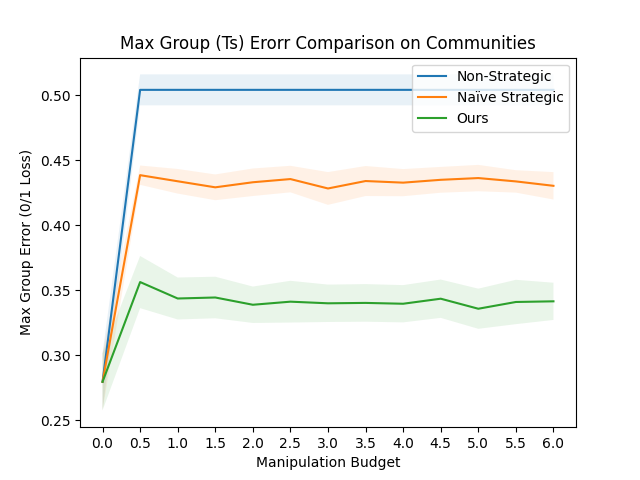}
        \label{fig:communities_equal}
    }
    \quad \quad
    \subfigure[]{
        \includegraphics[width=0.32\textwidth]{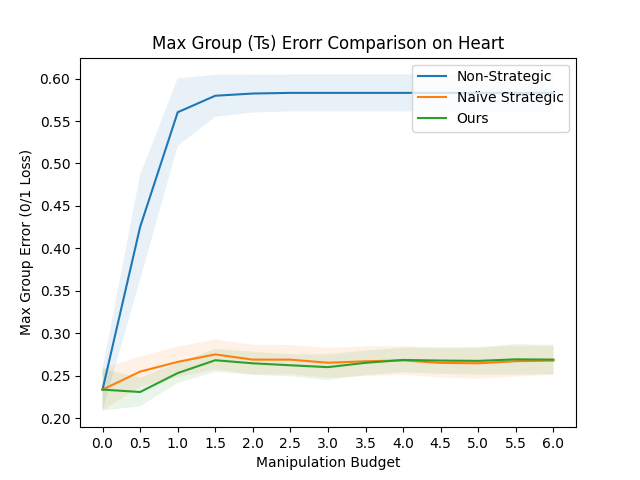}
        \label{fig:heart_equal}
    }   
    \caption{The max group error of the two baselines and our method, in the test time ({\tt Ts}), on the (a) COMPAS, (b) Credit, (c) Communities, and (d) Heart datasets with an equal budget $\tau$ for all agents.}
    \label{fig:equal-budget}
\end{figure*}

Our empirical evaluations show that the non-strategic approach performs the worst. In particular, this method performs, on average, $20\%$ worse on the COMAPS, Credit, and Communities datasets compared to the other two. Moreover, on the Heart dataset, this method performs, on average, $50\%$ worse.  

Furthermore, we observe that in the case of the same budget across groups, while our method has the best performance, its advantage over the na\"ive strategic approach is not significant---except on the Communities dataset.

Finally, we note that our theoretical results establish the convergence of our algorithms when the learning oracle $\text{WERM}_\Hs$ is used. However, in practice, we use the robust PRC heuristic instead of the perfect learning oracle, and therefore, convergence is not implied by the theory and must be examined experimentally. We present the convergence of our algorithm in Fig.~\ref{fig:convergence-equal} for $\tau = 1$ across all four datasets. The plots show that our algorithm have converged across all datasets despite using a heuristic instead of a perfect learning oracle.

\begin{figure*}[!ht]
    \centering
    \subfigure[]{
        \includegraphics[width=0.32\textwidth]{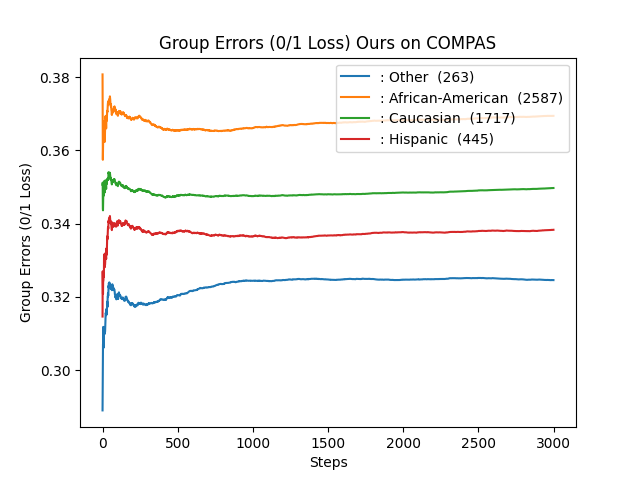}
        \label{fig:COMPAS_equal_convg}
    }
    \quad \quad
    \subfigure[]{
        \includegraphics[width=0.32\textwidth]{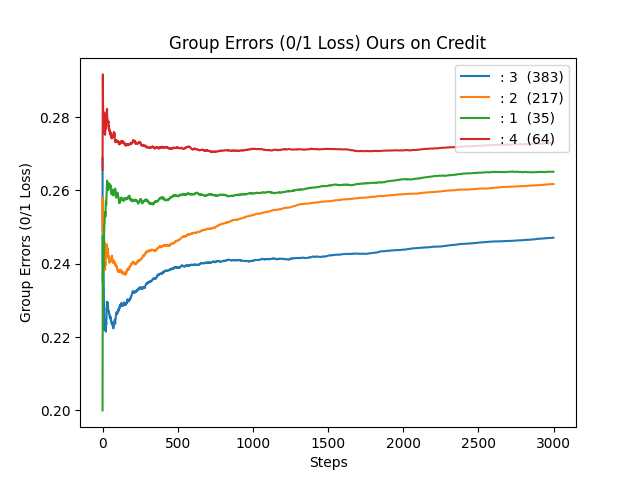}
        \label{fig:credit_equal_convg}
    }
    \quad \quad
    \subfigure[]{
        \includegraphics[width=0.32\textwidth]{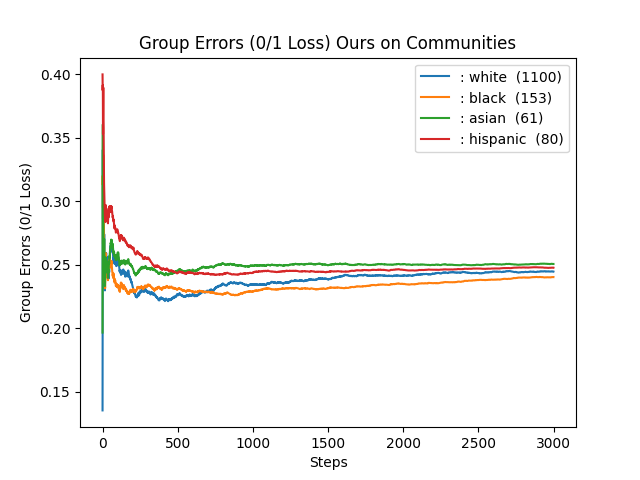}
        \label{fig:communities_equal_convg}
    }
    \quad \quad
    \subfigure[]{
        \includegraphics[width=0.32\textwidth]{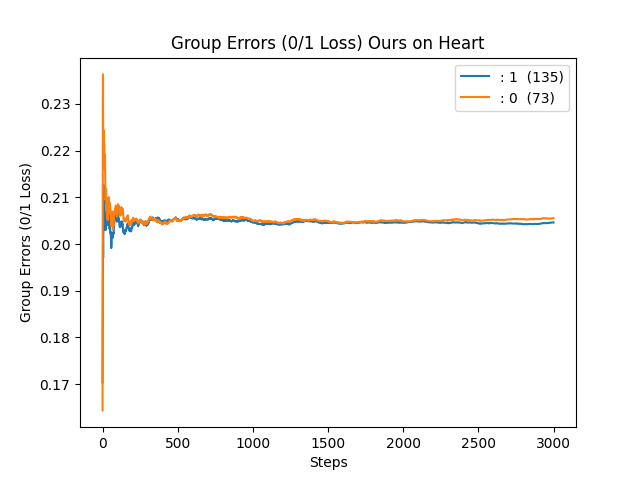}
        \label{fig:heart_equal_convg}
    }    \caption{The plots show the convergence of our algorithm during the training ({\tt Tr}) phase across (a) COMPAS, (b) Credit, (c) Communities, and (d) Heart, with $\tau = 1$, and the number of iterations is $T = 3000$.}
    \label{fig:convergence-equal}
\end{figure*}

\subsection{Different Manipulation Budgets Across Agent Groups}\label{sec:diff-group-budget}
Next, we consider varying manipulation budgets for each group of agents, denoted as the manipulation profile $\boldsymbol{\tau}$. We define the manipulation budget profile as $\boldsymbol{\tau} = (\tau_1, \cdots, \tau_G)$. 
More precisely, in our experiments, $\boldsymbol{\tau_f} = \tau \cdot \boldsymbol{f}$, where $\boldsymbol{f} \in [0,1]^G$ represents the fraction of the budget $\tau$ assigned to each group. For example, when $\tau = 2$ and $\boldsymbol{f} = (0.25, 1, 0, 0.5)$, the manipulation budget profile becomes $(0.5, 2, 0, 1)$, corresponding to the manipulation budgets for agents in groups indexed by 0, 1, 2, and 3, respectively.

In Fig.~\ref{fig:Heart_multi_group} to~\ref{fig:COMPAS_multi_group}, we present the performance of the baselines and our approach on the Heart, Credit, and COMPAS datasets. As observed in Fig.~\ref{fig:equal-budget}, on these datasets, our method and the na\"ive strategic baseline show similar performance when the manipulation budget is the same for all agents. However, when different groups have different manipulation budgets, we observe a clear dominance of our approach over the na\"ive strategic baseline, specifically as the value of $\tau$ increases. 

Across all datasets, in the first three plots, (e.g., in~\ref{fig:heart_(1,1)},~\ref{fig:heart_(1,0)},~\ref{fig:heart_(1,0.5)} for Heart), the manipulation budget profile is aligned with the group size, meaning that groups with a larger size are assigned a higher manipulation budget. This scenario corresponds to settings where minorities have a lower manipulation budget.
Additionally, we examine the reverse setting, where minority groups receive a higher manipulation budget (e.g., Fig.~\ref{fig:heart_(0.5,1)} for Heart). In this case, we observe that our approach remains robust, while the two baselines perform significantly worse for the minmax objective. Specifically, the na\"ive strategic baseline performs approximately $\times 2.17$, $\times 1.42$, and $\times1.17$ worse, as the value of $\tau$ increases to its max value, on the Heart, Credit, and COMPAS datasets, respectively, compared to our approach. 

In all plots, empirical confidence intervals are computed by repeating trials eight times with different seeds. The solid lines show the empirical mean of these runs for each budget and learner, and the shaded zone indicates a band of 95\% empirical confidence based on the standard deviation of the runs. 

\begin{figure*}[!ht]
    \centering
    \subfigure[$\boldsymbol{f}=(1,1)$.]{
        \includegraphics[width=0.32\textwidth]{Figures/Heart_MaxGroupError.png}
        \label{fig:heart_(1,1)}
    }
    \quad \quad
    \subfigure[$\boldsymbol{f}=(1,0)$]{
        \includegraphics[width=0.32\textwidth]{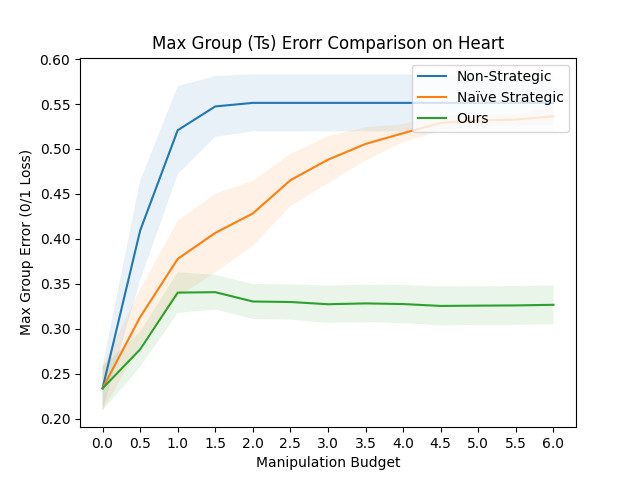}
        \label{fig:heart_(1,0)}
    }
    \quad \quad
    \subfigure[$\boldsymbol{f}=(1, 0.5)$]{
        \includegraphics[width=0.32\textwidth]{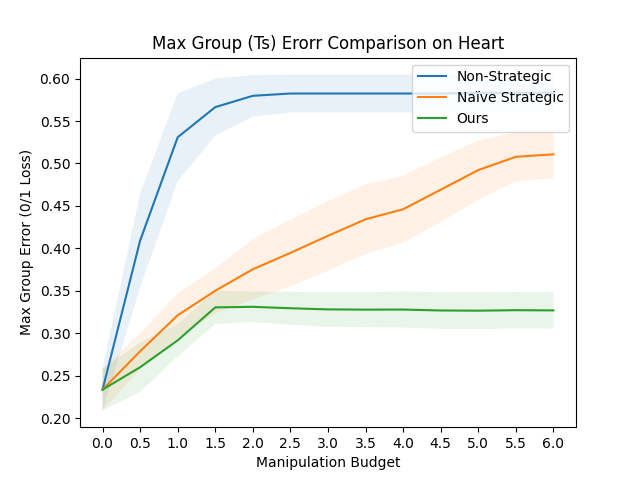}
        \label{fig:heart_(1,0.5)}
    }
    \quad \quad
    \subfigure[$\boldsymbol{f}=(0.5, 1)$]{
        \includegraphics[width=0.32\textwidth]{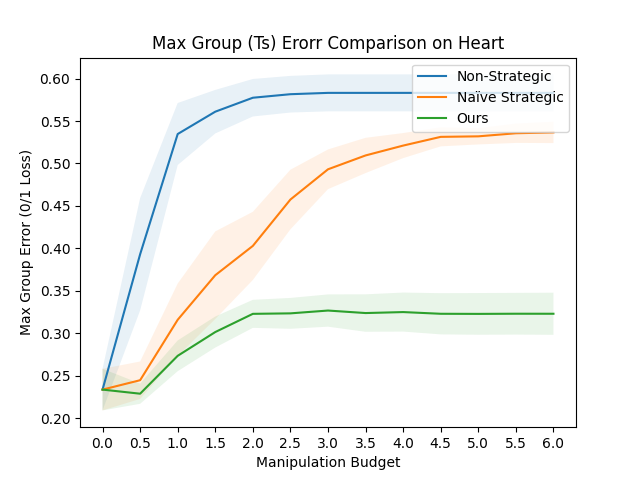}
        \label{fig:heart_(0.5,1)}
    }
    \caption{The performance of the baselines and our approach on the Heart dataset across different manipulation budget profiles for groups.}
    \label{fig:Heart_multi_group}
\end{figure*}


\begin{figure*}[!ht]
    \centering
    \subfigure[$\boldsymbol{f} = (1,1,1,1)$]{
        \includegraphics[width=0.32\textwidth]{Figures/Credit_MaxGroupError.png}
        \label{fig:credit_(1,1,1,1)}
    }
    \quad \quad
    \subfigure[$\boldsymbol{f} = (1, 1, 0, 0)$]{
        \includegraphics[width=0.32\textwidth]{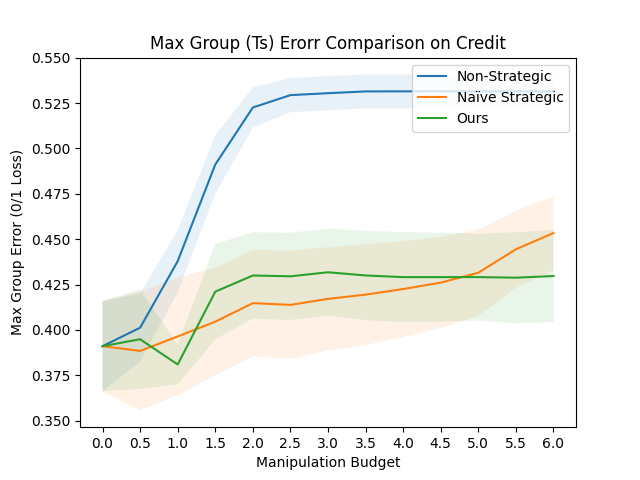}
        \label{fig:credit_(1, 1, 0, 0)_profile)}
    }
    \quad \quad
    \subfigure[$\boldsymbol{f} = (1, 0.5, 0, 0)$]{
        \includegraphics[width=0.32\textwidth]{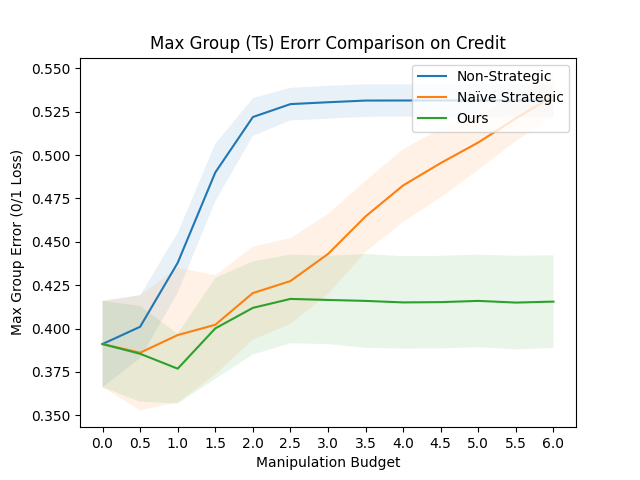}
        \label{fig:credit_(1, 0.5, 0, 0)_profile}
    }
    \quad \quad
    \subfigure[$\boldsymbol{f} = (0.25, 0.5, 1, 1)$]{
        \includegraphics[width=0.32\textwidth]{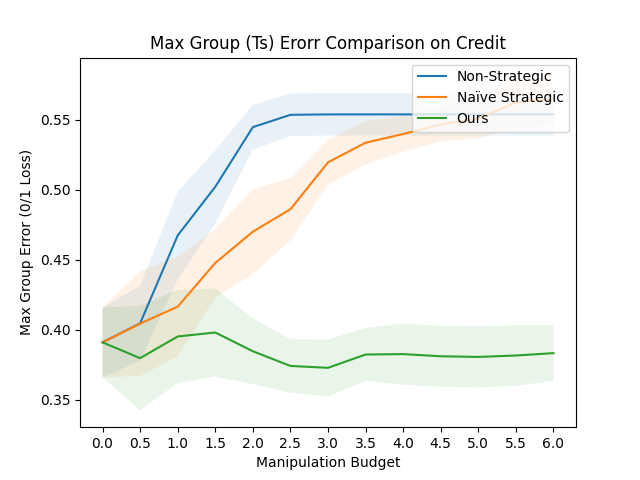}
        \label{fig:credit_(0.25,0.5,1,1)}
    }    
    \caption{The performance of the baselines and our approach on the Credit dataset across different values of manipulation budgets for groups.}
    \label{fig:Credit_multi_group}
\end{figure*}

\begin{figure*}[!ht]
    \centering
    \subfigure[$\boldsymbol{f} = (1,1,1,1)$.]{
        \includegraphics[width=0.32\textwidth]{Figures/COMPAS_MaxGroupError.png}
        \label{fig:COMPAS_(1,1,1,1)}
    }
    \quad \quad
    \subfigure[$\boldsymbol{f} = (0,1,0.75,0)$.]{
        \includegraphics[width=0.32\textwidth]{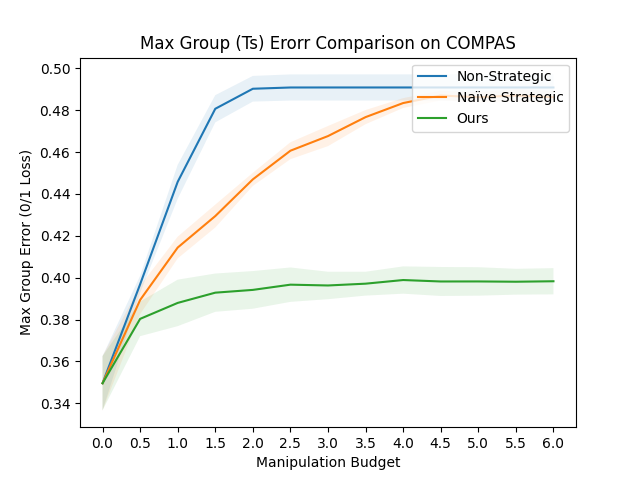}
        \label{fig:COMPAS_(0,1,0.75,0)}
    }
    \quad \quad
    \subfigure[$\boldsymbol{f} = (0.25, 1, 1, 0.25)$.]{
        \includegraphics[width=0.32\textwidth]{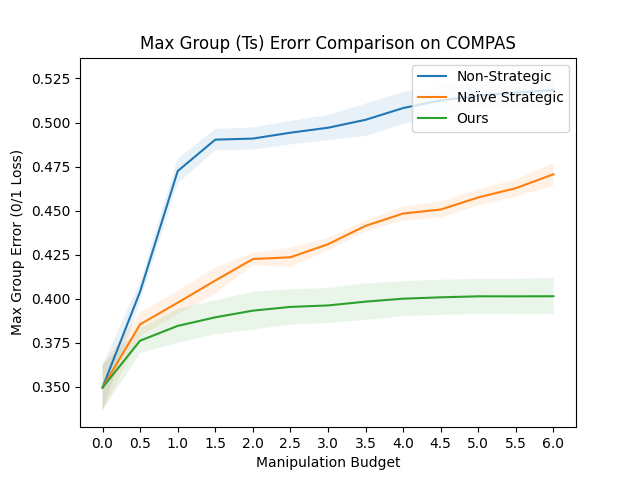}
        \label{fig:COMPAS_(1, 0.5, 0, 0)}
    }
    \quad \quad
    \subfigure[$\boldsymbol{f} = (1, 0.5, 0.5, 1)$.]{
        \includegraphics[width=0.32\textwidth]{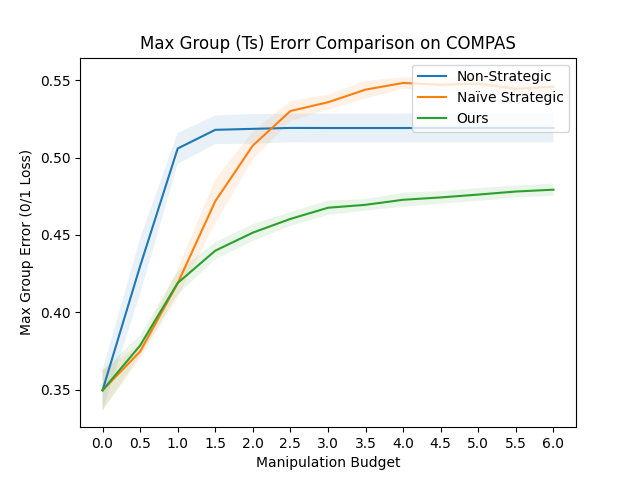}
        \label{fig:COMAPS_(1,0.5,0.5,1)}
    }

    \caption{The performance of the baselines and our approach on the COMPAS dataset across different values of manipulation budgets for groups.}
    \label{fig:COMPAS_multi_group}
\end{figure*}

\clearpage
\subsection{Experiments with Algorithm~\ref{alg:general2}}
Finally, we experiment with Algorithm~\ref{alg:general2} and visualize the trade-off between maximum group error and overall population error. For each experiment, we fix a dataset and group manipulation budgets, and then we run Algorithm 3 with 20 equally spaced values of $\gamma$ between 0 and 0.5 - $\hat{\gamma}$, thus allowing our maximal group error to potentially range from the minimax value $\hat{\gamma}$ to 0.5. For each value of $\gamma$, we run eight trials and compute the average across the eight trials for both the observed population error and observed maximal group error. We repeat this process for the two baselines. Rather than plotting the average results directly, we plot the Pareto curves to better visualize the trade-off between overall accuracy and fairness. 

We visualize one of these trade-off curves for each data set in Figure~\ref{fig:pareto}, below. Figure~\ref{fig:pareto_comm} shows trade-off curves on the Communities data set for each learner type with a manipulation budget of $\tau = 1$ and equal group profiles. Figure~\ref{fig:pareto_credit} shows a manipulation budget of 3 on the Credit data set with group budget profiles $(1, 0.5, 0, 0)$. Figure~\ref{fig:pareto_compas} shows several different models for each learner with an overall budget of 0.5 and group profiles $(0.25, 1, 1, 0.25)$, and in Figure~\ref{fig:pareto_heart} we observe a manipulation budget of 0.5 with group profiles $(0.5, 1)$ on the Heart data set. In this one, both strategic methods are fairly close together but clearly dominate the non-strategic approach. For Figures~\ref{fig:pareto_compas} and Figures~\ref{fig:pareto_credit} our method Pareto dominates the other two. However, it is worth noting that we do not observe a substantial range in models found for each learner in any of these settings.

\begin{figure*}[!ht]
    \centering
    \subfigure[$\boldsymbol{f} = (1,1,1,1)$.]{
        \includegraphics[width=0.3\textwidth]{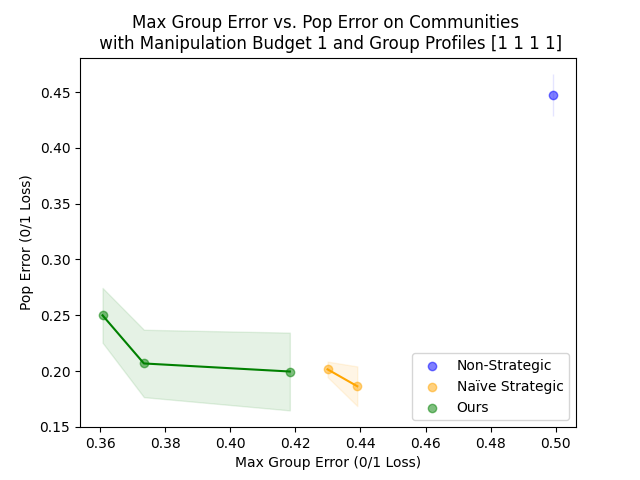}
        \label{fig:pareto_comm}
    }
    \quad \quad
    \subfigure[$\boldsymbol{f} = (1, 0.5, 0, 0)$]{
        \includegraphics[width=0.3\textwidth]{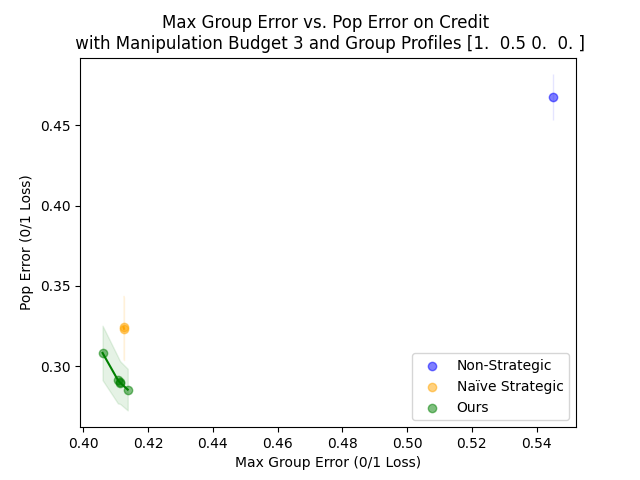}
        \label{fig:pareto_credit}
    }
    \quad \quad
    \subfigure[$\boldsymbol{f} = (0.25, 1, 1, 0.25)$.]{
        \includegraphics[width=0.3\textwidth]{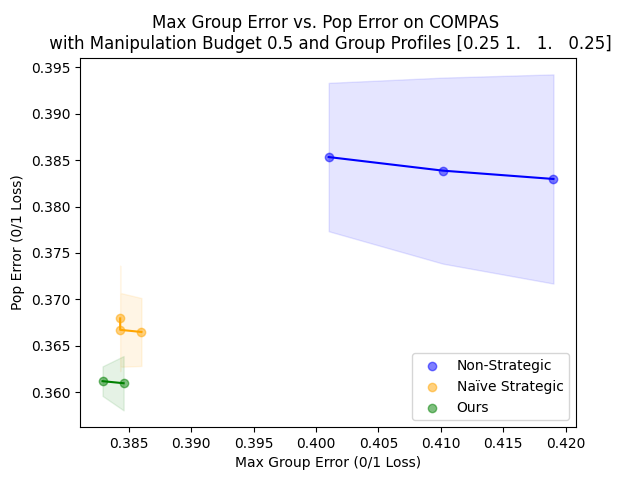}
        \label{fig:pareto_compas}
    }
    \quad \quad
    \subfigure[$\boldsymbol{f}=(0.5, 1)$]{
        \includegraphics[width=0.3\textwidth]{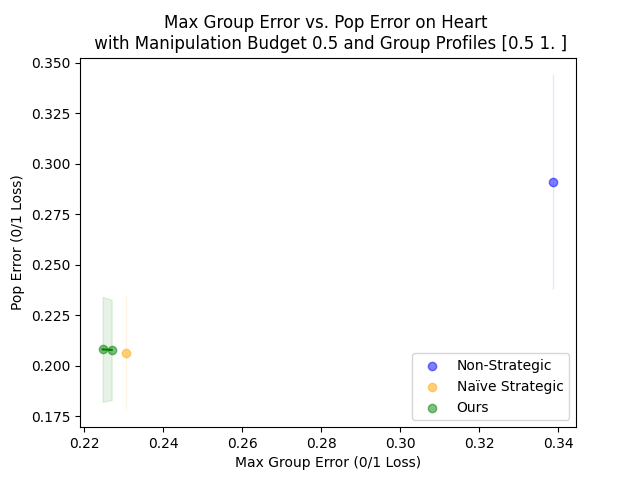}
        \label{fig:pareto_heart}
    }   
    \caption{The performance of baselines and our approach, using Algorithm~\ref{alg:general2}, with respect to a range of upper bounds on maximal group error on (a) Communities, (b) Credit, (c) COMPAS, and (d) Heart datasets.}
    \label{fig:pareto}
\end{figure*}

\section{Conclusion and Discussion}
In this paper, we have presented an algorithmic approach with provable guarantees and empirical backing to achieve minimax group fairness in certain strategic classification settings. We consider both the case where the cost functions are separable and the more involved case that relaxes this assumption. Our experiments present a heuristic that may be used in place of the learning oracle that our algorithms assume, and for the most part, we demonstrate this heuristic to be effective. Further work in this realm could extend our approach to handle other types of errors, such as false positives or real-valued losses, and our approach could also be adapted to handle intersectional group membership.

\bibliographystyle{plainnat}
\bibliography{arXiv-main}

\appendix

\section{Missing Proofs of Section~\ref{sec:sep}}\label{app:sep}

\objIsep*

\begin{proof}[Proof of Theorem~\ref{thm:sep}]
Recall that $ \ell_g (f_t)$ is the loss of group $g$ when using $f_t$ for $t = (t_1, \ldots, t_G)$ computed with respect to the distribution $D$, and $\hat{\ell}_g (f_t)$ is the corresponding loss computed with respect to the uniform distribution over the dataset $S$, i.e., the empirical distribution. Note that the first output of Algorithm~\ref{alg:separable} (for objective I) correctly computes \[
\min_{t \in T (S)} \max_{g \in \mathcal{G}} \hat{\ell}_g (f_t)
\]
because the second part of Assumption~\ref{ass:sep} allows us to write, for any $t$,
\[
f_t \left( \br (x_i,g_i,f_t) \right) = \prod_{g \in \mathcal{G}} \mathds{1} \left[ t_{g}^i \ge t_{g}\right]
\]
The next lemma shows that for every classifier $f_t$, there exists a classifier $f_{\bar{t}}$ with $\bar{t} \in T(S)$ such that they both induce the same labeling on the dataset. As an immediate consequence, this lemma implies that considering the thresholds in $T(S)$ suffices for computing the optimal classifier with respect to the empirical loss:
\[
\min_{t \in \mathbb{R}^G} \max_{g \in \mathcal{G}} \hat{\ell}_g (f_t) = \min_{t \in T(S)} \max_{g \in \mathcal{G}} \hat{\ell}_g (f_t)
\]

\begin{lemma}[Sufficiency of $T(S)$]\label{lem:thresholds}
    We have that for any dataset $S$, and any $t \in \mathbb{R}^G$, there exists a $\bar{t} \in T(S)$ such that $f_t$ and $f_{\bar{t}}$ induce the same labeling on the agents in $S$. In other words:
    \[
    \left\{ (x_i,g_i,y_i) \in S : f_{\bar{t}} \left( \br (x_i,g_i,f_{\bar{t}}) \right) = 1 \right\} = \left\{ (x_i,g_i,y_i) \in S : f_t \left( \br (x_i,g_i,f_t) \right) = 1 \right\}
    \]
\end{lemma}

\begin{proof}[Proof of Lemma~\ref{lem:thresholds}]
    Fix any dataset $S$. Recall that $T(S) = \prod_{g} T_g(S)$. Consider an arbitrary threshold $t = (t_1, \ldots, t_G) \in \mathbb{R}^G$ and let $\bar{t}$ be defined as follows:
    \[
    \bar{t} = (\bar{t}_1, \ldots, \bar{t}_G), \quad \bar{t}_g \triangleq \min \left\{ t'_g : t'_g \in T_g (S), t'_g \ge t_g \right\}
    \]
    Note that $\bar{t} \in T(S)$. By construction, if there exists an agent $(x_i,g_i,y_i) \in S$ that is classified as positive by $f_{\bar{t}}$, it is also classified as positive by $f_t$ because $f_{\bar{t}}$ only makes the positive region of $f_t$ smaller:
    \[
    \left\{ (x_i,g_i,y_i) \in S : f_{\bar{t}} \left( \br (x_i,g_i,f_{\bar{t}}) \right) = 1 \right\} \subseteq \left\{ (x_i,g_i,y_i) \in S : f_t \left( \br (x_i,g_i,f_t) \right) = 1 \right\}
    \]
    So it remains to show that if an agent $(x_i,g_i,y_i) \in S$ is classified as positive by $f_t$, it is also classified as positive by $f_{\bar{t}}$. Let $(x_i,g_i,y_i)$ be an agent such that $f_t \left( \br (x_i,g_i,f_t) \right) = 1$. We therefore must have that for all $g$, $t_g^i \ge t_g$ where $t_g^i$ is defined in Algorithm~\ref{alg:separable}. But $t_g^i \in T_g (S)$ which implies $t_g^i \ge \bar{t}_g$ for all $g$. Therefore
    \[
    f_{\bar{t}} \left( \br (x_i,g_i,f_{\bar{t}}) \right) = \prod_{g \in \mathcal{G}} \mathds{1} \left[ t_{g}^i \ge \bar{t}_{g}\right] = 1
    \]
\end{proof}

Next, we show that we can lift our empirical loss guarantees to ones that hold over the underlying distribution $D$ via a uniform convergence analysis.
\begin{lemma}[Generalization]\label{lem:gen}
    We have that with probability $1-\delta$ over $S \sim D^n$, for any group $g$,
    \[
    \sup_{t \in \mathbb{R}^G} \left| \hat{\ell}_g (f_t) -  \ell_g (f_t) \right| \le O \left( \sqrt{\frac{G \log (n) + \log \left( G / \delta \right)}{\min_g n_g}} \right)
    \]
\end{lemma}

\begin{proof}[Proof of Lemma~\ref{lem:gen}]
    
    Fix a group $g$ and define, for any $g' \in \mathcal{G}$,
    \[
    s_{g'}: \mathcal{X} \times \mathcal{G} \to \mathbb{R}, \quad s_{g'} (x,g) \triangleq \max_{z: c_{g} (x, z) < 1} b_{g'} (z)
    \]
    Observe that because of the second part of Assumption~\ref{ass:sep}, we can write
    \[
    f_t \left( \br (x,g,f_t) \right) = \prod_{g' \in \mathcal{G}} \mathds{1} \left[ s_{g'} (x,g) \ge t_{g'}\right]
    \]
    Therefore, we have that
    \[
    \ell_g (f_t) = \Pr_{(x,y) \sim D_g} \left[ \prod_{g' \in \mathcal{G}} \mathds{1} \left[ s_{g'} (x,g) \ge t_{g'}\right] \neq y \right], \quad \hat{\ell}_g (f_t) = \Pr_{(x,y) \sim S_g} \left[ \prod_{g' \in \mathcal{G}} \mathds{1} \left[ s_{g'} (x,g) \ge t_{g'}\right] \neq y \right]
    \]
    where $S_g$ is the data set containing only members of group $g$. Define the function class $\mathcal{F}_g = \{ \tilde{f}^g_{t}: t = (t_1, \ldots, t_G) \in \mathbb{R}^G\}$ such that
    \[
    \tilde{f}^g_{t} (x) \triangleq \prod_{g' \in \mathcal{G}} \mathds{1} \left[ s_{g'} (x,g) \ge t_{g'}\right]
    \]
    Hence, we can re-write
    \[
    \ell_g (f_t) = \Pr_{(x,y) \sim D_g} \left[ \tilde{f}^g_{t} (x) \neq y \right], \quad \hat{\ell}_g (f_t) = \Pr_{(x,y) \sim S_g} \left[ \tilde{f}^g_{t} (x) \neq y \right]
    \]
    Note that the VC dimension (Definition~\ref{def:vc}) of class $\mathcal{F}_g$ is bounded $VC(\mathcal{F}_g) \le G$ because it contains only $G$ dimensional threshold functions in $\reals^G$. Therefore, using standard uniform convergence guarantees for VC classes (Theorem~\ref{thm:vc-gen}), we have that with probability $1-\delta$ over the draw of $S$,
    \[
    \sup_{t \in \mathbb{R}^G} \left| \hat{\ell}_g (f_t) - \ell_g (f_t) \right| \le O \left( \sqrt{\frac{G \log (n_g) + \log \left( 1 / \delta \right)}{n_g}} \right)
    \]
    where $n_g = |S_g|$.
    Therefore, using a union bound, with probability $1-\delta$ over $S$, for any group $g$, we have
    \begin{align*}
    \sup_{t \in \mathbb{R}^G} \left| \hat{\ell}_g (f_t) - \ell_g (f_t) \right|
    \le O \left( \sqrt{\frac{G \log (n) + \log \left( G / \delta \right)}{\min_g n_g}} \right)
    \end{align*}
\end{proof}

Lemma~\ref{lem:thresholds}, together with Lemma~\ref{lem:gen}, Lemma~\ref{lem:F}, and the sample complexity bound of the theorem prove our result. To elaborate, let $f_{\hat{t}}$ be the classifier returned by Algorithm~\ref{alg:separable} (for objective I). We have that with probability $1-\delta$,
\begin{align*}
\max_g \ell_g (f_{\hat{t}}) 
\le \max_g \hat{\ell}_g (f_{\hat{t}}) + \frac{\gamma}{2} 
= \min_{t \in T(S)} \max_g \hat{\ell}_g (f_t) + \frac{\gamma}{2}
= \min_{t \in \mathbb{R}^G} \max_g \hat{\ell}_g (f_t) + \frac{\gamma}{2}
&\le \min_{t \in \mathbb{R}^G} \max_g  \ell_g (f_t) + \gamma \\
&\le \min_{h \in \mathcal{H}} \max_g  \ell_g (h) + \gamma
\end{align*}
Here, the first and the second inequalities are due to Lemma~\ref{lem:gen}, and the last inequality is due to Lemma~\ref{lem:F}. The first equality is due to the construction of the algorithm, and the second equality is due to Lemma~\ref{lem:thresholds}.
\end{proof}

\objIIsep*

\begin{proof}[Proof of Theorem~\ref{thm:sep2}]
    First, we prove the minimax fairness guarantee of the theorem. We have that with probability at least $1-\delta$,

\begin{align*}
\max_g \ell_g (f_{\hat{t}}) 
\le \max_g \hat{\ell}_g (f_{\hat{t}}) + \frac{\epsilon}{2}
\le  \min_{t \in T(S)} \max_g \hat{\ell}_g (f_t) + \gamma +  \epsilon + \frac{\epsilon}{2}
&= \min_{t \in \mathbb{R}^G} \max_g \hat{\ell}_g (f_t) + \gamma +  \frac{3\epsilon}{2}\\
&\le \min_{t \in \mathbb{R}^G} \max_g  \ell_g (f_t) + \gamma + 2 \epsilon \\
&\le \min_{h \in \mathcal{H}} \max_g \ell_g (h) + \gamma + 2\epsilon
\end{align*}
where the first and the third inequalities are due to Lemma~\ref{lem:gen} and the sample complexity bound of the theorem. The second inequality follows from the construction of the algorithm, and the last inequality is due to Lemma~\ref{lem:F}. The equality follows from Lemma~\ref{lem:thresholds}.

For the accuracy guarantees, let us define for any $\gamma$, the set of threshold classifiers that satisfy $\gamma$-minimax fairness with respect to the distribution ($C_\gamma$) and the dataset ($\hat{C}_\gamma$):
\[
C_\gamma \triangleq \left\{ t \in \mathbb{R}^G: \max_g {\ell}_g(f_t) \le \min_{t \in \mathbb{R}^G} \max_g {\ell}_g(f_t) +  \gamma \right\}
\]

\[
\hat{C}_\gamma \triangleq \left\{ t \in \mathbb{R}^G: \max_g \hat{\ell}_g(f_t) \le \min_{t \in \mathbb{R}^G} \max_g \hat{\ell}_g(f_t) +  \gamma \right\}
\]

\begin{lemma}\label{lem:1}
    We have that with probability at least $1 - \delta$, $C_\gamma \subseteq \hat{C}_{\gamma + \epsilon}$.
\end{lemma}

\begin{proof}[Proof of Lemma~\ref{lem:1}]
    This follows from our generalization guarantee in Lemma~\ref{lem:gen}. Let $t \in C_\gamma$. We have that with probability $1-\delta$,
\begin{align*}
\max_g \hat{\ell}_g (f_{t}) &\le \max_g \ell_g (f_t) + \frac{\epsilon}{2} 
\le \min_{t \in \mathbb{R}^G} \max_g \ell_g (f_t) + \gamma + \frac{\epsilon}{2} 
\le \min_{t \in \mathbb{R}^G} \max_g \hat{\ell}_g (f_t) + \gamma + \epsilon
\end{align*}
Therefore, $t \in \hat{C}_{\gamma + \epsilon}$ and this completes the proof.
\end{proof}

Let us resume the proof for our accuracy guarantees. We have that with probability at least $1-\delta$,

\begin{align*}
\ell (f_{\hat{t}}) 
\le \hat{\ell} (f_{\hat{t}}) + \frac{\epsilon}{2} 
=  \min_{t \in T(S) \cap \hat{C}_{\gamma + \epsilon}} \hat{\ell} (f_t) +  \frac{\epsilon}{2}
= \min_{t \in \hat{C}_{\gamma + \epsilon}} \hat{\ell} (f_t) +  \frac{\epsilon}{2}
\le \min_{t \in {C}_{\gamma}} \hat{\ell} (f_t) +  \frac{\epsilon}{2}
&\le \min_{t \in {C}_{\gamma}} {\ell} (f_t) +  \epsilon \\
&= \text{OPT} (\mathcal{F}, \gamma) +  \epsilon \\
&\le \text{OPT} (\mathcal{H}, \gamma) + \epsilon
\end{align*}
where the first and the third inequalities are due to Lemma~\ref{lem:gen} and the sample complexity bound of the theorem; the second inequality is due to Lemma~\ref{lem:1} and the last one follows from Lemma~\ref{lem:F}. The first equality is due to the construction of the algorithm, and the second equality is due to Lemma~\ref{lem:thresholds}. The third equality follows from the definition of $\text{OPT} (\mathcal{F}, \gamma)$.
\end{proof}

\section{Missing Proofs of Section~\ref{sec:gen}}\label{app:gen}
\subsection{Objective I}

\optminmax*
\begin{proof}
    Observe that
    \begin{align*}
        \max_{g \in \mathcal{G}} \hat{\ell}_g( \hat{p} ) 
        \le \max_{\lambda \in \Lambda} \sum_g \lambda^g  \hat{\ell}_g( \hat{p} )
        \le \sum_g \hat{\lambda}^g \hat{\ell}_g( \hat{p} ) + \nu
        \le \min_{p \in \Delta (\mathcal{H} (S))} \sum_g \hat{\lambda}^g \hat{\ell}_g( p ) + 2\nu 
        &\le \min_{p \in \Delta (\mathcal{H} (S))} \max_{g \in \mathcal{G}} \hat{\ell}_g(p) + 2 \nu \\
        & = \min_{p \in \Delta (\mathcal{H})} \max_{g \in \mathcal{G}} \hat{\ell}_g(p) + 2 \nu 
    \end{align*}
    Here, the second and the third inequalities follow from $(\hat{p}, \hat{\lambda})$ being a $\nu$-approximate equilibrium pair.
\end{proof}

\objIgen*

\begin{proof}[Proof of Theorem~\ref{thm:general1}]
    for any strategies of the players $p \in \Delta (\mathcal{H})$ and $\lambda \in \Lambda$, let
    \[
    U(p, \lambda) \triangleq \sum_g \lambda_g \cdot  \hat{\ell}_g(p)
    \]
    denote the objective value of the game. Note that the regret of the two players jointly satisfy:
    \[
    \sum_{t=1}^T U(h_t, \lambda_t) - \min_{p \in \Delta (\mathcal{H})} \sum_{t=1}^T U(p, \lambda_t) \le 0,
    \quad
    \max_{\lambda \in \Lambda} \sum_{t=1}^T U(h_t, \lambda) - \sum_{t=1}^T U(h_t, \lambda_t)  \le \sqrt{\frac{T \log G}{2}}
    \]
    The first follows from the fact that the learner best responds in every round of the algorithm, and the second inequality is simply the regret of the Exponential Weights algorithm for appropriately choesn learning rate of $\eta = \sqrt{8 \log G / T}$. (see \citet{cesa1999prediction}). We therefore have from Theorem~\ref{thm:noregret} that the average play of the players $(\hat{p} = \frac{1}{T} \sum_t h_t, \hat{\lambda} = \frac{1}{T} \sum_t \lambda_t)$ forms a $\nu$-approximate equilibrium of the game where
    \[
    \nu = \sqrt{\frac{\log G}{2T}}
    \]
    Therefore, Lemma~\ref{lem:optimality}, as well as the choice of $T$ in the algorithm, imply that
    \begin{equation}\label{eq:emp}
    \max_{g \in \mathcal{G}} \hat{\ell}_g( \hat{p} ) \le \min_{p \in \Delta (\mathcal{H})} \max_{g \in \mathcal{G}} \hat{\ell}_g(p) + \frac{\gamma}{2}
    \end{equation}
    Next, we use uniform convergence guarantees to lift our empirical guarantees to ones that hold over the distribution $D$.
    \begin{lemma}[Generalization]\label{lem:gen2}
    We have that for any $\delta$, with probability at least $1-\delta$ over the $i.i.d.$ draws of $S \sim D^n$, for any group $g$,
    \[
    \sup_{p \in \Delta (\mathcal{H})} \left| \ell_g( p ) - \hat{\ell}_g( p )\right| \le O \left( \sqrt{\frac{d_\mathcal{H} \log (n) + \log \left( G / \delta \right)}{\min_g n_g}} \right)
    \]
    \end{lemma}
    \begin{proof}[Proof of Lemma~\ref{lem:gen2}]
        Fix a group $g$. Define the function class $\mathcal{F}_g = \{ f^g_h: h \in \mathcal{H }\}$ where
        \[f^g_h: \mathcal{X} \to \mathcal{Y}, \, f^g_h (x) = h\left( \br (x,g,h) \right)
        \]
        Observe that for any $h \in \mathcal{H}$,
        \[
        \ell_g (h) = \Pr_{(x,y) \sim D_g} \left[ h\left( \br (x,g,h) \right) \neq y \right] = \Pr_{(x,y) \sim D_g} \left[ f^g_h (x) \neq y \right]
        \]
        \[
        \hat{\ell}_g (h) = \Pr_{(x,y) \sim S_g} \left[ h\left( \br (x,g,h) \right) \neq y \right] = \Pr_{(x,y) \sim S_g} \left[ f^g_h (x) \neq y \right]
        \]
        Using standard uniform convergence guarantees for VC classes (Theorem~\ref{thm:vc-gen}), we get that with probability at least $1 - \delta$ over the draw of $S$,
        \[
        \sup_{h \in \mathcal{H}} \left| \ell_g( h ) - \hat{\ell}_g( h )\right| \le O \left( \sqrt{\frac{VC(\mathcal{F}_g) \log (n_g) + \log \left( 1 / \delta \right)}{n_g}} \right)
        \]
        where $VC(\mathcal{F}_g)$ is the VC dimension of $\mathcal{F}_g$. But $VC(\mathcal{F}_g) = d_\mathcal{H}$ by the definition of strategic VC dimension (Definition~\ref{def:svc}). Therefore, we have that with probability at least $1-\delta$,
        \[
        \sup_{h \in \mathcal{H}} \left| \ell_g( h ) - \hat{\ell}_g( h )\right| \le O \left( \sqrt{\frac{d_\mathcal{H} \log (n_g) + \log \left( 1 / \delta \right)}{n_g}} \right)
        \]
        Note that for a randomized classifier $p \in \Delta (\mathcal{H})$, because of linearity of expectation we have
        \[
        \ell_g( p ) - \hat{\ell}_g( p ) = \mathbb{E}_{h \sim p} \left[ \ell_g (h)\right] - \mathbb{E}_{h \sim p} \left[ \hat{\ell}_g (h)\right] =  \mathbb{E}_{h \sim p} \left[ \ell_g (h) - \hat{\ell}_g (h) \right]
        \]
        Therefore, with probability at least $1 - \delta$,
        \begin{align*}
        \sup_{p \in \Delta (\mathcal{H})} \left| \ell_g( p ) - \hat{\ell}_g( p )\right| \le \sup_{p \in \Delta (\mathcal{H})} \mathbb{E}_{h \sim p} \left| \ell_g( h ) - \hat{\ell}_g( h )\right|
        = \sup_{h \in \mathcal{H}} \left| \ell_g( h ) - \hat{\ell}_g( h )\right|
        \le O \left( \sqrt{\frac{d_\mathcal{H} \log (n_g) + \log \left( 1 / \delta \right)}{n_g}} \right)
        \end{align*}
        Here, the first inequality is an application of Jensen's inequality, and the equality follows from the linearity of expectation. A union bound over the $G$ groups completes the proof.
    \end{proof}
    We are now ready to finish the proof of the theorem.
    We have that with probability at least $1-\delta$,
    \begin{align*}
        \max_{g \in \mathcal{G}} \ell_g( \hat{p} ) \le \max_{g \in \mathcal{G}} \hat{\ell}_g( \hat{p} ) + \frac{\gamma}{4}
        \le \min_{p \in \Delta (\mathcal{H})} \max_{g \in \mathcal{G}} \hat{\ell}_g(p) + \frac{3\gamma}{4}
        \le \min_{p \in \Delta (\mathcal{H})} \max_{g \in \mathcal{G}} \ell_g(p) + \gamma
    \end{align*}
    where the first and last inequalities follow from Lemma~\ref{lem:gen2} and the sample complexity bound of the theorem. The second inequality follows from our empirical guarantees (Equation~(\ref{eq:emp})).
\end{proof}

\subsection{Objective II}
\opterror*
\begin{proof}[Proof of Lemma~\ref{lem:optimality2}]
    For any $x \in \mathbb{R}$, let $x_+ \triangleq \max (x, 0)$. We first show that the pair $(\hat{p}, \hat{\lambda})$ satisfies:
    \begin{equation}\label{eq:dualbound}
        \sum_g \hat{\lambda}^g \left( \hat{\ell}_g (\hat{p}) - \hat{\gamma} - \gamma \right) \ge B \max_g \left( \hat{\ell}_g (\hat{p}) - \hat{\gamma} - \gamma \right)_+ - \nu
    \end{equation}
    To show this, let $\lambda$ be the best response of the dual player to $\hat{p}$:
    \[
    \lambda = \begin{cases}
        0 = (0,0, \ldots, 0) \in \Lambda & \max_g \hat{\ell}_g (\hat{p}) \le \hat{\gamma} + \gamma \\
        B e_{g^\star} & \max_g \hat{\ell}_g (\hat{p}) > \hat{\gamma} + \gamma
    \end{cases}
    \]
    where $g^\star \in \argmax_g \left( \hat{\ell}_g (\hat{p}) - \hat{\gamma} - \gamma \right)$ and $e_i$ is defined as the $i$th vector of the standard basis of $\mathbb{R}^G$, for any $i$. Note that in this case the dual player puts all its mass $B$ on the most violated constraint. Now we have that
    \[
    \mathcal{L} ( \hat{p}, \hat{\lambda}) = \hat{\ell} ( \hat{p}) + \sum_g \hat{\lambda}^g \left( \hat{\ell}_g( \hat{p} ) - \hat{\gamma} - \gamma\right)
    \]
    \[
    \mathcal{L} ( \hat{p}, {\lambda}) = \hat{\ell} ( \hat{p} ) + \sum_g \lambda^g \left( \hat{\ell}_g( \hat{p} ) - \hat{\gamma} - \gamma\right) = \hat{\ell} ( \hat{p} ) + B \max_g \left( \hat{\ell}_g (\hat{p}) - \hat{\gamma} - \gamma \right)_+
    \]
    But because $(\hat{p}, \hat{\lambda})$ is a $\nu$-approximate equilibrium of the game, we have
    \[
    \mathcal{L} ( \hat{p}, \hat{\lambda}) \ge \mathcal{L} ( \hat{p}, {\lambda}) - \nu
    \]
    which proves Equation~(\ref{eq:dualbound}).

    We are now ready to prove the Lemma. Let $p$ be any feasible solution to the optimization problem~(\ref{eq:opthat}), i.e. one that satisfies: $\max_g \hat{\ell}_g (\hat{p}) \le \hat{\gamma} + \gamma$. We have that
    \[
    \mathcal{L} ( p, \hat{\lambda}) = \hat{\ell} (p) + \sum_g \hat{\lambda}^g \left( \hat{\ell}_g(p) - \hat{\gamma} - \gamma\right) \le \hat{\ell} (p)
    \]
    Therefore, because $(\hat{p}, \hat{\lambda})$ is a $\nu$-approximate equilibrium of the game,
    \[
    \mathcal{L} ( \hat{p}, \hat{\lambda}) \le \mathcal{L} ( p, \hat{\lambda}) + \nu \le \hat{\ell} (p) + \nu
    \]
    On the other hand, using Equation~(\ref{eq:dualbound}), we have
    \[
    \mathcal{L} ( \hat{p}, \hat{\lambda}) = \hat{\ell} ( \hat{p}) + \sum_g \hat{\lambda}^g \left( \hat{\ell}_g( \hat{p} ) - \hat{\gamma} - \gamma\right) \ge \hat{\ell} ( \hat{p}) + B \max_g \left( \hat{\ell}_g (\hat{p}) - \hat{\gamma} - \gamma \right)_+ - \nu \ge \hat{\ell} ( \hat{p}) - \nu
    \]
    Putting these inequalities together, we have that
    \[
    \hat{\ell} ( \hat{p}) \le \hat{\ell} (p) + 2\nu
    \]
    which proves the first part of the Lemma because $p$ can be \emph{any} feasible solution of the problem~(\ref{eq:opthat}). To prove the second part, note that we can use the same inequalities to obtain
    \[
     B \max_g \left( \hat{\ell}_g (\hat{p}) - \hat{\gamma} - \gamma \right)_+ \le \hat{\ell} (p) - \hat{\ell} ( \hat{p}) + 2 \nu \le 1 + 2 \nu
    \]
    Therefore
    \[
    \max_g \left( \hat{\ell}_g (\hat{p}) - \hat{\gamma} - \gamma \right) \le \max_g \left( \hat{\ell}_g (\hat{p}) - \hat{\gamma} - \gamma \right)_+ \le \frac{1 + 2 \nu}{B}
    \]
    which completes the proof.
\end{proof}

\objIIgen*
\begin{proof}[Proof of Theorem~\ref{thm:general2}]
    Note that the regret of the two players jointly satisfy:
    \[
    \sum_{t=1}^T \mathcal{L} (h_t, \lambda_t) - \min_{p \in \Delta (\mathcal{H})} \sum_{t=1}^T \mathcal{L} (p, \lambda_t) \le 0,
    \quad
    \max_{\lambda \in \Lambda} \sum_{t=1}^T \mathcal{L} (h_t, \lambda) - \sum_{t=1}^T \mathcal{L} (h_t, \lambda_t)  \le \frac{B^2 \sqrt{T}}{2} + G \left( \sqrt{T} - \frac{1}{2} \right)
    \]
    The first follows from the fact that the learner best responds in every round of the algorithm, and the second inequality is the regret of the Online Projected Gradient Descent (PGD) algorithm for appropriately choesn learning rate of $\eta_t = t^{-1/2}$. (see \citep{zinkevich2003online}). We therefore have from Theorem~\ref{thm:noregret} that the average play of the players $(\hat{p} = \frac{1}{T} \sum_t h_t, \hat{\lambda} = \frac{1}{T} \sum_t \lambda_t)$ forms a $\nu$-approximate equilibrium of the game where
    \[
    \nu \le \left( \frac{B^2}{2} + G \right) \frac{1}{\sqrt{T}} = \frac{\epsilon}{4}
    \]
    where we substitute the values of $T$ and $B$ from our algorithm. Therefore, Lemma~\ref{lem:optimality2} implies that
    \begin{equation}\label{eq:emp2}
    \hat{\ell} (\hat{p}) \le \widehat{\text{OPT}} \left( \Delta (\mathcal{H}), \gamma + \epsilon \right) + \frac{\epsilon}{2}, \quad
    \max_{g \in \mathcal{G}} \hat{\ell}_g( \hat{p} ) \le \left( \hat{\gamma} + \gamma +  \epsilon \right) + \frac{\epsilon}{2}
    \end{equation}
    We note that the algorithm is run with extra slack of $\epsilon$ for its bound on minimax fairness ($\gamma + \epsilon$ instead of $\gamma$). This is why we have $\widehat{\text{OPT}} \left( \Delta (\mathcal{H}), \gamma + \epsilon \right)$ in the bound. This extra slack is crucial for us to convert guarantees with respect to the empirical optimum to ones that compete with the distributional ${\text{OPT}} \left( \Delta (\mathcal{H}), \gamma \right)$, as we show in the following lemma.
    \begin{lemma}\label{lem:opt2}
        We have that with probability at least $1 - \delta$ over the draw of $S$,
    \[
    \widehat{\text{OPT}} \left( \Delta (\mathcal{H}), \gamma + \epsilon \right) \le {\text{OPT}} \left( \Delta (\mathcal{H}), \gamma \right) + \frac{\epsilon}{4}
    \]
    \end{lemma}
    \begin{proof}[Proof of Lemma~\ref{lem:opt2}]
        Let us define, for any $\gamma$,
        \[
        C_\gamma \triangleq \left\{ p \in \Delta (\mathcal H): \max_g {\ell}_g(p) \le \min_{p \in \Delta (\mathcal H)} \max_g {\ell}_g(p) +  \gamma \right\}
        \]

        \[
        \hat{C}_\gamma \triangleq \left\{ p \in \Delta (\mathcal H): \max_g \hat{\ell}_g(p) \le \hat{\gamma} +  \gamma \right\}
        \]
        We show that with probability at least $1-\delta$, $C_\gamma \subseteq \hat{C}_{\gamma + \epsilon}$. To show this, take a $p \in C_\gamma$. Observe that
        \begin{align*}
         \max_g \hat{\ell}_g(p) 
         \le \max_g {\ell}_g(p) + \frac{\epsilon}{4} 
            \le \min_{p \in \Delta (\mathcal H)} \max_g {\ell}_g(p) + \gamma + \frac{\epsilon}{4}
            \le \min_{p \in \Delta (\mathcal H)} \max_g \hat{\ell}_g(p) + \gamma + \frac{\epsilon}{2} 
            \le \hat{\gamma} + \gamma + \epsilon
        \end{align*}
        Therefore $p \in \hat{C}_{\gamma + \epsilon}$. Here, the first inequality follows from our generalization guarantees (Lemma~\ref{lem:gen2}) and the sample complexity bound of the theorem. The second inequality follows from the fact that $p \in C_\gamma$, the third is another application of Lemma~\ref{lem:gen2}, and the last one follows because $\hat{\gamma} \ge \min_{p \in \Delta (\mathcal H)} \max_g \hat{\ell}_g(p)$, i.e., the estimated minmax value from Algorithm~\ref{alg:general1} is always greater than or equal to the true minmax value on the dataset.
        
        Next, we have, with probability at least $1-\delta$,
        \begin{align*}
             \widehat{\text{OPT}} \left( \Delta (\mathcal{H}), \gamma + \epsilon \right) &= \min_{p \in \Delta (\mathcal{H})} \left\{ \hat{\ell} (p): \max_{g \in \mathcal{G}} \hat{\ell}_g(p) \le \hat{\gamma} + \gamma + \epsilon \right\} \\
             &= \min_{p \in \hat{C}_{\gamma + \epsilon}} \hat{\ell} (p) \\
             &\le \min_{p \in C_\gamma} \hat{\ell} (p) \\
             &\le \min_{p \in C_\gamma} {\ell} (p) + \frac{\epsilon}{4} \\
             &= \min_{p \in \Delta (\mathcal{H})} \left\{ \ell (p): \max_{g \in \mathcal{G}} \ell_g(p) \le \min_{p' \in \Delta (\mathcal{H})} \max_{g \in \mathcal{G}} \ell_g(p') + \gamma \right\} + \frac{\epsilon}{4} \\
             &= {\text{OPT}} \left( \Delta (\mathcal{H}), \gamma \right)+ \frac{\epsilon}{4}
        \end{align*}
        where the first inequality follows from the fact that $C_\gamma \subseteq \hat{C}_{\gamma + \epsilon}$, and the second is an application of Lemma~\ref{lem:gen2} and the sample complexity bound of the theorem.
        \end{proof}
        
        We are now ready to complete the proof of Theorem~\ref{thm:general2}. We start by proving the error guarantee of the theorem. We have that with probability at least $1-\delta$,
        \begin{align*}
         \ell (\hat{p}) 
         &\le \hat{\ell} (\hat{p}) + \frac{\epsilon}{4} 
         \le \widehat{\text{OPT}} \left( \Delta (\mathcal{H}), \gamma + \epsilon \right) + \frac{3\epsilon}{4} 
         \le {\text{OPT}} \left( \Delta (\mathcal{H}), \gamma \right)+ \epsilon
        \end{align*}
        where the first inequality follows from Lemma~\ref{lem:gen2} and the sample complexity bound of the theorem, the second follows from Equation~(\ref{eq:emp2}), and the last one follows from Lemma~\ref{lem:opt2}. 
        For fairness guarantees of the theorem, observe that with probability at least $1-\delta$,
        \begin{align*}
            \max_g \ell_g (\hat{p} ) \le \max_g \hat{\ell}_g (\hat{p} ) + \frac{\epsilon}{4} 
            \le \hat{\gamma} + \gamma +  \frac{7 \epsilon}{4}
            = \max_g \hat{\ell}_g (\tilde{p}) + \gamma +  \frac{7 \epsilon}{4}
            &\le \max_g {\ell}_g (\tilde{p}) +  \gamma +  2\epsilon \\
            &\le \min_{p \in \Delta (\mathcal{H})} \max_g \ell_g (p) + \gamma +  3\epsilon
        \end{align*}
        Here, the first inequality follows from Lemma~\ref{lem:gen2} and the sample complexity bound of the theorem, the second follows from Equation~(\ref{eq:emp2}). Let us call the output of Algorithm~\ref{alg:general1} that we use in Algorithm~\ref{alg:general2} by $\tilde{p}$: $\hat{\gamma} = \max_g \hat{\ell}_g (\tilde{p})$. The third inequality follows from the generalization guarantees in Lemma~\ref{lem:gen2}, and the last one follows from the minmax guarantee of $\tilde{p}$ from Theorem~\ref{thm:general1}.
\end{proof}

\section{Implementation of the Heuristic in Experiments}\label{app:exp}

In our experiments, we use the paired regression classifiers (PRC) heuristic, used previously in~\citep{gerrymandering,agarwal2018reductions,diana2021minimax} in the non-strategic setting. In our strategic setting, we modify this heuristic by shifting the linear classifier that it outputs.

The PRC has the notable feature of requiring the solution of a convex optimization problem, even in the presence of negative sample weights.
\begin{definition}[Paired Regression Classifier \cite{gerrymandering}]
\label{def:prc}
Given a vector of sample weights $\{ w_i \}_{i=1}^n$, the paired regression classifier operates as follows: We form two weight vectors, $z^0$ and $z^1$, where $z^k_i$ corresponds to the penalty assigned to sample $i$ in the event that it is labeled $k$. For the correct labeling of $x_i$, the penalty is $0$. For the incorrect labeling, the penalty is the current sample weight of the point, $w_i$. We fit two linear regression models $h^0$ and $h^1$ to predict $z^0$ and $z^1$, respectively, on all samples. Then, given a new point $x$, we calculate $h^0(x)$ and $h^1(x)$ and output $h(x) = \argmin_{k\in\{0,1\}} h^k(x)$.
\end{definition}
In our implementation, we translate the costs so that the cost of labeling any example with 0 is 0, and therefore we need only train $h^1$ to predict the cost of predicting 1. The sign of $h^1(x)$ then allows us to classify $x$.

\end{document}